%% file: clc.tex
\documentclass[journal]{IEEEtran}
\usepackage{amsmath,amssymb,amsfonts,amsthm}
\usepackage{algorithmicx}
\usepackage{algorithm}
\usepackage{array}
\usepackage{textcomp}
\usepackage{stfloats}
\usepackage{url}
\usepackage{verbatim}
\usepackage{cite}
\usepackage{tikz}
\usepackage{tikz-cd}
\usepackage{pgf}
\usepackage{floatflt}
\usepackage{hyperref}
\usepackage{siunitx} 

\usepackage[utf8]{inputenc}
\usepackage[english]{babel}
\usepackage[numbers]{natbib}

\usepackage{bm}
\usepackage{gensymb}
\usepackage[noend]{algpseudocode}
\usepackage{tabularx}

\usepackage{subcaption}
\usepackage{adjustbox}
\usepackage{booktabs}
\usepackage{multirow}

\IEEEoverridecommandlockouts

\input{custom_commands.tex}

\begin{document}
	
	\title{Lane Detection using Graph Search and Geometric Constraints for Formula Student Driverless}
	\author{
		\IEEEauthorblockN{Ivo Ivanov, Carsten Markgraf}
		\thanks{Ivo Ivanov and Carsten Markgraf are with the Faculty of Electrical Engineering at the Technical University Augsburg, Augsburg, Germany. Corresponding Author Email: \code{ivo.ivanov@hs-augsburg.de}}  
	}
	
	\maketitle
	
	\IEEEpeerreviewmaketitle 
	
	\begin{abstract}
	 Lane detection is a fundamental task in autonomous driving. While the problem is typically formulated as the detection of continuous boundaries, we study the problem of detecting lane boundaries that are sparsely marked by 2D points with many false positives. This problem arises in the Formula Student Driverless (FSD) competition and is challenging due to its inherent ambiguity. Previous methods are inefficient and unable to find long-horizon solutions. We propose a deterministic algorithm called CLC that uses backtracking graph search with a learned likelihood function to overcome these limitations. We impose geometric constraints on the lane candidates to guarantee a geometrically sound lane. Our exhaustive search leads to finding the global optimum in 45\% of instances, and the algorithm is overall robust to up to 50\% false positives. Our algorithm runs in less than 15 ms on a single CPU core, meeting the low latency requirements of autonomous racing. We extensively evaluate our method on real data and realistic racetrack layouts, and show that it outperforms the state-of-the-art by detecting long lanes over 100 m with few (0.6\%) critical failures. This allows our autonomous racecar to drive close to its physical limits on a previously unknown racetrack without being limited by perception. We release our dataset with realistic Formula Student racetracks to enable further research.
	
	\end{abstract}
	
	
	\section{Introduction}
	\begin{figure}[!h]
		\centering
		\includegraphics[width=0.99\linewidth]{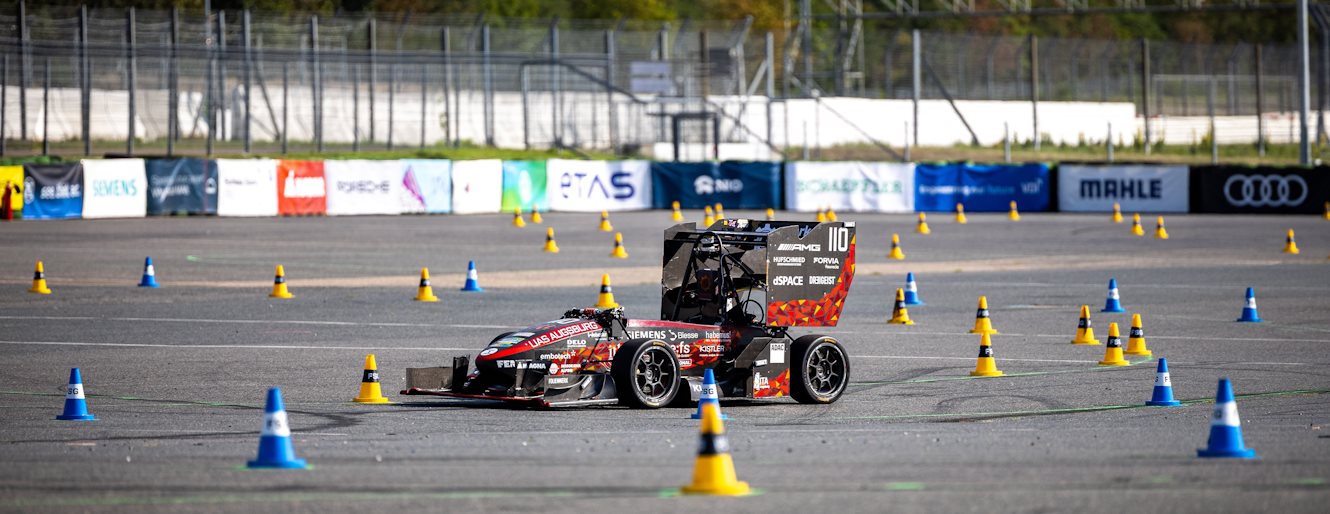}
		\caption{The autonomous race car of the team StarkStrom Augsburg at the Formula Student Germany competition 2023 on the race track marked by blue and yellow traffic cones. (Copyright FSG Lodholz).}
		\label{fig:rokofsg2023}
	\end{figure}
	
	Autonomous vehicle racing is a challenging application of autonomous driving, featuring competitions like the Indy Autonomous Challenge, the F1TENTH competition, and the Formula Student Driverless (FSD) competition \cite{9790832}. Autonomous racing poses numerous challenges to the system and algorithms employed. Due to the high speeds, a high degree of uncertainty originates from short perception intervals, at the same time a low-latency perception and processing is required. When driving at high speeds up to 20 m/s (72 km/h) on an unknown track in the FSD-competition, reliable lane detection is crucial for optimal planning of the racing line to achieve a minimum lap time. 
	
	This paper focuses on lane detection using only 2D-points, a problem arising in the FSD competition where the lane is marked with traffic cones (Fig. \ref{fig:rokofsg2023}). Other potential applications of the approach presented in this paper besides the FSD competition may include detecting lanes on road construction sites \cite{6629537} and making the lane detection more robust to unreliable perception as well as resolving ambiguity arising due to damaged lane markings.
	
	The FSD competition is a student engineering competition on autonomous racing using self-developed race vehicles. Among several disciplines in this competition, the \textit{Autocross} is the most challenging one
	for the perception and planning systems 
	as the track layout is not known in advance. In this discipline, the vehicle must drive one lap autonomously on a closed circuit in the minimum time without leaving the track.
	
	As typical for autonomous robots, the perception pipeline of an FSD race car first detects cones with LiDAR-sensors and cameras and then builds a map of the environment with SLAM-algorithms using these detected cones, usually as a set of 2D points \cite{9341702, Kabzan2019AMZDT, alvarez2022software}, a landmark map. After mapping, the objective is to detect the two boundaries in this set of points defining the lane.
	
	\begin{figure}[!tb]
		\centering	
			\begin{adjustbox}{width=1.\linewidth}
				\includegraphics{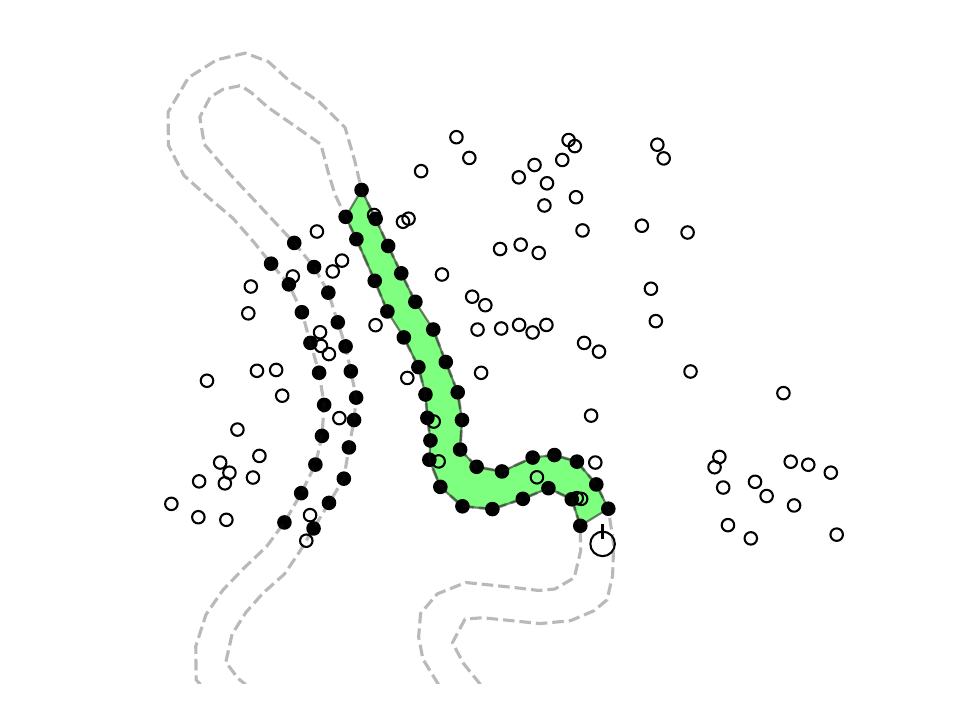}
				\includegraphics{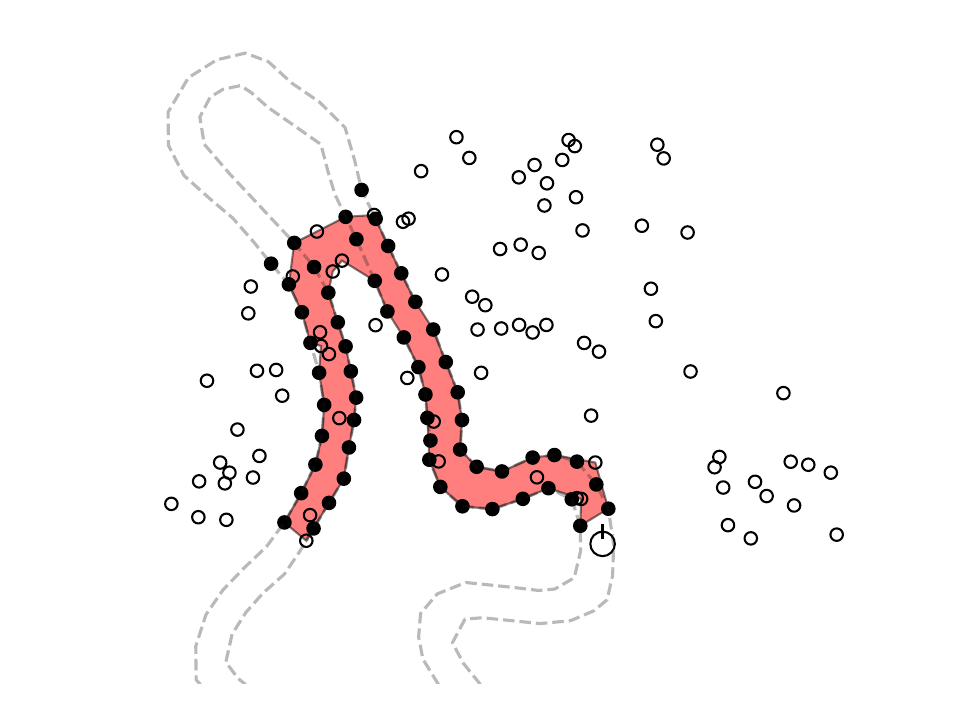}
			\end{adjustbox}
		\caption{Detecting the driving lane (outlined in gray) on race tracks where the boundaries are marked only sparsely by points and with many false-positives (empty circles) introduces ambiguity: On the left is the correct prediction, but given only the partial observation of the boundary points, it may be detected incorrectly as shown on the right, taking a shortcut. } 
		\label{fig:ambiguity-problem}
	\end{figure}

	Lane detection is a well-studied task in autonomous driving where continuous curves are detected in camera images or LiDAR point clouds \cite{4580573, 9576699}.
	In the FSD competition however, the problem is quite different as the boundaries are marked by points as opposed to continuous lines. Also, lanes of racetracks are usually much more curvy and also the high number of false-positives introduces ambiguities (see Fig. \ref{fig:ambiguity-problem}).
	
	Previous methods are only able to detect short sections of the track, also they are prone to false-positives and inefficient. Proposed methods based on mixed-integer optimization are robust but do not scale well with growing number of input points due to long solve times and have thus only a short horizon \cite{9197098}. The motivation of our work is therefore the question: \textit{Can we detect reliably a driving lane sparsely marked by points and many false-positives over a long-horizon, enabling an FSD-race car to utilize its full performance ?}
    
	We propose a deterministic lane detection algorithm named Cone Lane Connector (CLC), which uses a modified depth-first graph search combined with geometric constraint checking to find candidate solutions. This step makes our approach not only robust by guaranteeing geometrically sound lanes, but also efficient by allowing pruning of the search space through backtracking. In a second step, we use a machine learning approach to find the most likely candidate lane by ranking.
	Due to the exhaustive search with effective heuristics, our method is able to cope with a high number of false positives. Surprisingly, although the combinatorial solution space of two paths in a graph is prohibitively large, enumeration becomes feasible when imposing these geometric constraints. This allows us to find the global optimum in many cases. The machine learning approach to ranking also eliminates the need for hand-tuning heuristics.
	
	We summarize the contributions of this work as follows:
	\begin{enumerate}
		\item Using heuristically guided graph search that enumerates simultaneously two paths in a graph (path pairs) for finding feasible lanes with boundaries marked by points (Alg. \ref{alg:enumerate_path_pairs_rec})
		\item Carefully designed geometric constraints which guarantee lane soundness and allow for an effective pruning of the search space with backtracking
		\item Using a neural network for ranking to find the most likely lane
		\item An open dataset with real formula student race-track layouts including annotated ground-truth driving lanes 
	\end{enumerate}

	We released our dataset with racetrack layouts at \footnote{\href{https://github.com/iv461/fsd\_racetrack\_dataset}{https://github.com/iv461/fsd\_racetrack\_dataset}}. A video demonstrating the overall performance of our autonomous race car driving at speeds up to 20 m/s (72 km/h) on a previously unknown race track using the proposed lane detection is available at: \href{https://youtu.be/9MWKDJeAEDU}{https://youtu.be/9MWKDJeAEDU}
	
	\section{Related Work}
	\label{sec:related}
	
	\subsection{Lane detection from boundaries marked by points}
	
	Although lane detection is a well-studied topic in autonomous driving, only a few works address the problem of detecting lanes marked by points, most of them are directly related to the FSD-competition.
	In \cite{Kabzan2019AMZDT}, Kabzan et al. describe an approach for lane detection for the FSD competition based on exhaustive search with breadth-first search (BFS) and beam search. They construct the search graph via Delaunay triangulation and formulate the likelihood of candidate lanes using handcrafted geometric features. They do not use hard constraints, therefore geometrically sound solutions are only favored instead of enforced. This work is improved in \cite{9341702} where Andresen et al. propose to use Bayesian inference approach for deriving the likelihood of a lane based on the color of all observed cones and hand-tuned probabilities. A detection horizon of up to 15m is reported which is rather short and limits the planning horizon.
	
	In \cite{9197098}, Brandes et al. also study lane detection for the FSD-competition by framing it as a constraint optimization problem in the domain of two paths in a graph. They formulate geometric constraints and use a handcrafted geometric likelihood function to select the most likely candidate. As a solving strategy, they formulate the problem as a constrained binary integer optimization problem to be able to use off-the-shelf solvers (e.g. \textit{Mosek}). While their declarative approach is elegant, it suffers from inefficiency -- it scales with the number of points in the map input instead of the number of feasible candidate solutions. This requires limiting the number of input points to only 12-16 in order to meet the real-time requirements, which makes the overall detection horizon rather short.
	
	In \cite{6629537}, Tanzmeister et al. study lane detection of a road lane marked with traffic cones, which is the same problem as it arises in the FSD-competition.
	Instead of detecting the lane boundaries directly, a path inside the lane is detected. The authors use the path planning algorithms \textit{Rapidly-exploring Random Trees} (RRT) and A* as a search strategy and impose kinematic constraints.
	Although RRT seems to work for this lane detection problem where the boundaries are marked by points, path planning has the different objective of finding a continuous curve in a continuous search space \cite{lavalle2001rapidly, vonasek2009rrt, Noakes2007GeometryFR}. The problem of detecting boundaries from a set of points on the other hand has a much smaller solution set. Thus, the solution set is oversampled, making the algorithm conceptually inefficient. In fact, as the authors note: "Many of the paths are similar and thus lead to the same road boundary".
	
	\subsection{Lane detection}
	
	More commonly, lane detection is performed by detecting continuously marked lane boundaries.
	Typical lane detection algorithms deal with computer vision aspects such as invariance to lighting and noise, classical approaches such as \cite{WANG2000677, 4580573, 638604, kaske1997lane} use camera images and first detect edges, then use Hough voting to estimate the most likely parameters for splines modeling the boundaries. From multiple solutions, the most likely one is selected by handcrafted heuristics. A different approach is proposed in \cite{6856551} where the authors use a graph search in the image space to detect unstructured roads. They make use of geometric prior information such as the minimum and maximum lane width and combine it with other handcrafted heuristics.
	
	Deep-learning based approaches such as \cite{Zhang_2018_ECCV, qin2020ultra, Tabelini_2021_CVPR} use neural networks for either boundary detection or segmentation of the lane area.
	Geometric properties known a-priori like straightness and small lane width variance have been used to improve lane detection in \cite{9857364}. The idea of using geometric constraints to enforce geometrically sound lanes has been proposed in \cite{10172163} and incorporated as a post-processing step instead for making the search more efficient. Using handcrafted geometric features based on the variance has been used \cite{6594920} to select the most likely lane candidate via a binary classifier.

	\subsection{Path planning}
	
	Another related line of work are path planning algorithms used for mobile robots and autonomous vehicles. Path planning searches for a single path within a continuous search space (2D or 3D). This space may also be a traversable area, possibly with obstacles, and the task is to find a path from a starting point to an end point that avoids the obstacles while satisfying kinematic constraints \cite{Noakes2007GeometryFR}. In addition, the path planning algorithm usually optimizes for path length, curvature, etc. \cite{Gasparetto2015}. In \cite{10.1007/978-3-642-00196-3_8}, the authors propose an algorithm based on A* for real-time path planning in an urban environment, using heuristics-guided graph search to find kinematically feasible solutions and a subsequent smoothing step. In autonomous racing, path planning is approached via trajectory optimization given an already detected lane \cite{global-traj-optimization-tum}. 
	\subsection{Graph search techniques}

	For many combinatorial search and optimization problems with a prohibitively large solution set, specialized search techniques have been proposed.
	Search heuristics are used to find good locally optimal solutions, for this machine learning based methods using graph neural networks (GNNs) have been proposed in \cite{8968113} to improve graph algorithms. Even for fundamental graph problems it has been shown that learned heuristics can outperform handcrafted heuristics \cite{NEURIPS2018_8d3bba74}.
		
	\section{Overview} 
	
	\begin{algorithm}
		\caption{CLC: Lane detection in a map of 2D points.}
		\label{alg:clc}
		\begin{algorithmic}[1]
			\Require $\ConeMap$: set of 2D points representing the map, 
			$\CarPose$: pose of the car in map coordinate system
			\Ensure Most likely driving lane
			\Function{\code{CLC}}{$\ConeMap, \CarPose$}
			\State $\SearchGraph \leftarrow \code{construct\_search\_graph}(\ConeMap, \MaxConeDistance)$ \label{clc:construct-search-graph}
			\State $s_l, s_r \leftarrow \code{find\_starting\_vertices}(\ConeMap, \CarPose)$ \label{clc:find-starting-cones}
			\State \% Enumerate path pairs representing lanes (Alg. \ref{alg:enumerate_path_pairs_rec})
			\State candidates $\leftarrow \code{EPP}(\SearchGraph, (s_l, s_r))$ 
			\State \Return{\code{most\_likely\_candidate}(candidates)}
			\EndFunction
		\end{algorithmic}
	\end{algorithm}
	
	\subsection{Generative model of the map} 
	The perception system of the autonomous vehicle first builds a map of the environment. This map is the input to the lane detection algorithm and is typically obtained in 2D, as the height deviation is negligible. We assume the following generative model of the map. The map $\ConeMap$ consists of measured points $p_1, p_2, ..., p_N \in \Rtwo$ which are noisy. We assume that the points are perturbed by Gaussian noise. The measurement error which is typically achieved in such maps produced by LiDAR-sensors is 0.2m - 0.3m \cite{9341702}.
	
	Additionally, there may be false-positive points, i.e. not every point in the map may actually be present in the environment. We define the false-positive rate (FP-Rate) as the fraction of false-positive points of the total number of points in the map. If the point is a false-positive, we assume it is sampled from a uniform distribution in the range of the perception field. We measured values from 10\% to 30\% for the false-positive rate of our perception pipeline, which are similar to the 15\% - 25\% reported in \cite{9197098}.

	Every cone also has a color, indicating whether it belongs to the left or right boundary, but our algorithm does not use this information. We observed that in our perception pipeline, the color information is unreliable especially at large distances, and also wrong colors are strongly correlated, violating the independence assumption made in \cite{9341702}.
		
	\subsection{Problem statement and algorithm outline}
	Given a map of 2D points generated by the generative model, the problem of lane detection is to find two sequences of map points that represent the boundaries of the lane. Additionally, the algorithm should only output geometrically sound lanes, where soundness is defined by a set of constraints.
	As an additional input besides the map points, we use the pose of the car, which indicates the beginning of the lane and acts a hint to its direction. It is necessary to disambiguate which lane to detect, since several (partial) lanes may be visible on the map.
	
	Algorithm \ref{alg:clc} outlines the proposed approach. It receives the map points $\ConeMap$ and the car pose $\CarPose$ as an input. As the boundaries can be described naturally by two paths in a graph, we use a graph as the problem domain. After constructing the search graph, we enumerate feasible candidate solutions using path pair enumeration. In the second step, we select the most likely candidate. Our proposed algorithm is deterministic and also complete due to exhaustive enumeration.

	In the following sections, we describe each step of the proposed algorithm.
	
	\section{Graph search for lane detection}
	
	\subsection{Search graph construction}
	\label{sec:search-graph-construction}
	\begin{figure}[!t]
		\centering
		\begin{minipage}{0.49\textwidth}
			\begin{adjustbox}{width=\linewidth,center}
				\includegraphics{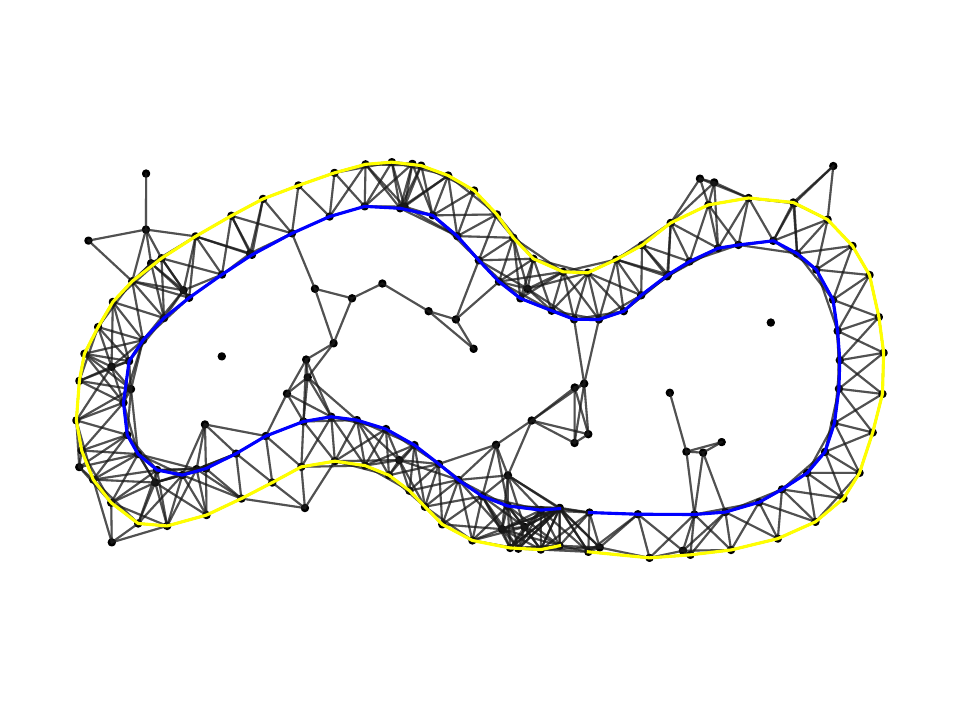}
			\end{adjustbox}
			\subcaption[second caption.]{Proposed search graph construction}\label{fig:search-graph-comparison-proposed}
		\end{minipage}\\
		\begin{minipage}{0.49\textwidth}
			\begin{adjustbox}{width=\linewidth,center}
				\includegraphics{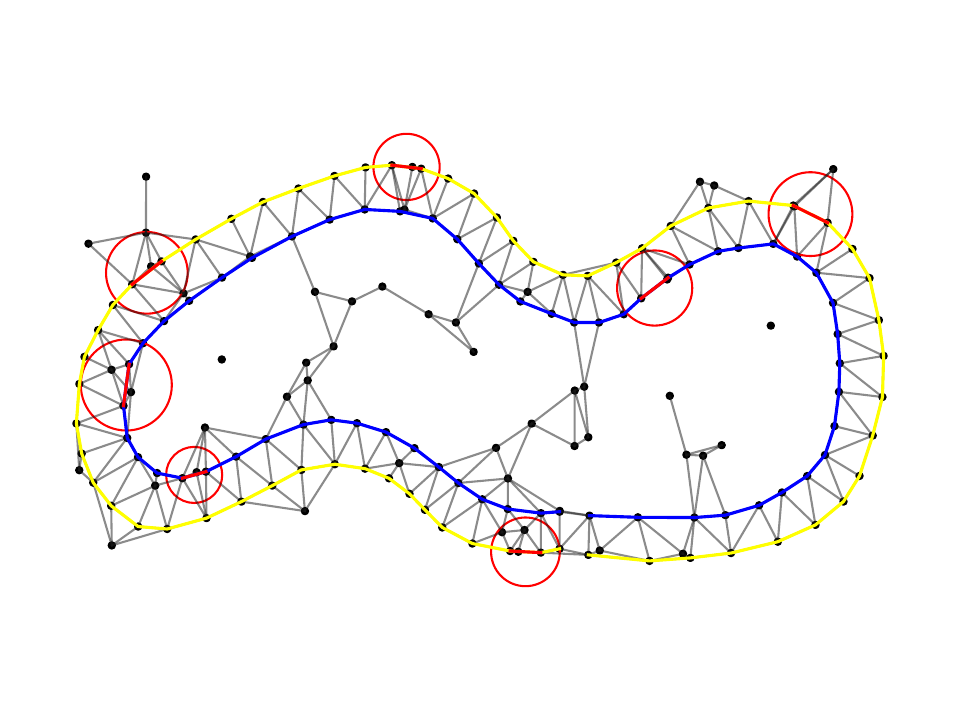}
			\end{adjustbox}
			\subcaption[first caption.]{Search graph construction using the Delaunay-triangulation}\label{fig:search-graph-comparison-delaunay}
			
		\end{minipage}%
		\caption{We propose to construct the search graph simply by adding an edge between all points closer than $d_{max}$ (e.g. 5m) (a). The Delaunay-triangulation (b) is an alternative way for constructing the search-graph, but in case of false-positives, the graph has missing edges (highlighted in red) that belong to the correct solution. Thus, it is not possible to find the correct solution in this search graph. (the false-positive rate shown here is 30\%).} 
		\label{fig:search-graph-comparison}
	\end{figure}
	
	In the first step, we construct a search graph from the map with the vertices representing every map point. Edges are inserted between every pair of vertices which can be connected in a boundary. A first constraint $\MaxConeDistance$ is used here, which states the maximum distance two points can have for a connection. Specifically for the FSD-competition, this constraint is stated in a rulebook \cite{fsgHandbook2023} to be 5m, i.e. two consecutive points of a boundary cannot be further away than 5m.
	The search graph $\SearchGraph = (\Vertices, \Edges)$ is therefore a simple and undirected graph constructed from the map $\ConeMap$ consisting of 2D-points $p_1, ..., p_N \in \Rtwo$. It is defined by its vertex- and edges-sets $V, E$ as:
	\begin{align}
		\Vertices &= \{1, ..., N\}\\
		\Edges &= \{ \{i, j\} \in \Vertices^2 \mid ||p_{i}- p_{j}||_2 \leq \MaxConeDistance, i \neq j \}
	\end{align}
	An example of a search graph constructed in this way is illustrated in Fig. \ref{fig:search-graph-comparison-proposed}.\\
	
	\subsection{Delaunay-triangulation for search graph construction} 
	\label{sec:do-not-use-delaunay}
	The Delaunay-triangulation has been previously proposed for construction of the search graph \cite{Kabzan2019AMZDT, alvarez2022software, 9352547}. It offers the advantage that it has fewer edges than our proposed approach. Using our proposed approach, the number of edges grows quadratically in the number of vertices. On the other hand, graphs constructed by triangulation have at most $3n -6$ edges for $n$ vertices, linear in the number of vertices \cite{delaunay-mesh-book}. However, the main disadvantage of the search graph constructed by triangulation is that it does not necessarily contain all edges of the correct driving lane in case of false positives (Fig. \ref{fig:search-graph-comparison-delaunay}). Thus, it is generally not possible to find the correct lane in this search graph, as it is not contained in it. In other words, the correct solution is immediately excluded by the construction of the search graph. We will show in section \ref{sec:exe-time-eval} that in practice, the theoretically lower number of edges does not lead to a significantly lower runtime of our algorithm. Our approach does not require the heuristic geometric plausibility provided by the search graph, since it explicitly uses geometric constraints.
	
	\subsection{Determining two starting vertices}
	To perform the graph search for lane detection, we determine initially two left and right starting vertices.
	Determining these vertices is a non-critical process and can be done via simple geometric rules which are implemented in the function \code{find\_starting\_vertices} \ref{clc:find-starting-cones}: Given the car pose in the map coordinate system, the two starting vertices  $s_l, s_r \in \Vertices$ have to be close to the car, thus we conduct a search within a maximum radius such as 2 m. Additionally, the left starting vertex must have a positive angle, the right a negative angle relative to the car heading vector. From multiple pairs that meet these criteria, we select the one that is the most symmetrical with respect to the line defined by the position and direction vector of the car.
			
	\subsection{Enumeration of path pairs}
	\label{enumeration}
	
	\begin{algorithm}[!h]
		\caption{Enumerate path pairs which satisfy constraints.}
		\label{alg:enumerate_path_pairs_rec}
		\begin{algorithmic}[1]
			\Require{$\SearchGraph$: adjacency list of the graph; $\PathPair$: current path pair, initially  $( (s_l), (s_r) )$ (containing only the starting vertices); $V$: pair of adjacency lists with visited vertices, initially $(\{\}, \{\})$; $i$: iteration counter, initially 0;
				$it_{max}$: iteration limit (default $it_{max} =\DefaultMaxIterations{}$)
			}
			\Ensure Set of path pairs which satisfy constraints
			\hypertarget{ref:epp-func2}{
				\Function{\code{EPP}}{$\SearchGraph, P, V, i$}}
			\State $\EnumeratedPathPairs \leftarrow \EmptySet$ \% Set of all found path pairs
			\Loop  \% For all adjacent vertices
			\If{$i \geq it_{max}$}
			\State \Return{$\EnumeratedPathPairs$}
			\EndIf
			\State $i \leftarrow i + 1$ 
			\State $c_s \leftarrow P[s].\textrm{back()} \textrm{ for } s \in \{0,1\}$ \% curr. verts.
			\State $v_s \leftarrow V[s][c_s] \textrm{ for } s \in \{0,1\}$ \% visited verts.
			\State \% Adjacent unvisited vertices 
			\State $u_s \leftarrow (\SearchGraph[c_s] \setminus v_s) \setminus P[s] \textrm{ for } s \in \{0,1\}$
			\If{$u_0 = \EmptySet \land u_1 = \EmptySet$} 
			\State \Return{$\EnumeratedPathPairs$} 
			\EndIf
			\State \% Choose next vertices for both sides
			\State $n_s \leftarrow \code{NVD}(P[s], u_s) \textrm{ for } s \in \{0,1\}$ \ref{sec:nvd-def}
			\If{$u_0 \neq \EmptySet \land u_1 \neq \EmptySet$} 
			\State $\NextSide \leftarrow \code{LRD}(P, n_0, n_1)$ \% Choose next side \ref{sec:lr-decider}
			\Else{ $\NextSide \leftarrow 1 \textrm{ if } u_0 = \EmptySet \textrm{ else } \NextSide \leftarrow 0$ }
			\EndIf
			\State $P[\NextSide].\textrm{push}(n_{\NextSide})$ \% Add next vertex to path
			\State $V[\NextSide][c_{\NextSide}].\textrm{add}(n_{\NextSide})$ \% Mark as visited
			\If{\code{CD}$(P)$} \% If lane satisfies constraints \ref{eq:constraint-decider-def}
			\State $\EnumeratedPathPairs \leftarrow \EnumeratedPathPairs  \cup \{P\} $
			\EndIf
			\If{$\lnot$\code{BTD}$(P, u_0 \neq \EmptySet, u_1 \neq \EmptySet)$} \ref{sec:backtracking-decider-btd}
			\State $\EnumeratedPathPairs \leftarrow \EnumeratedPathPairs  \cup \code{EPP}(G, P, V, i)$ 
			\EndIf
			\State $P[\NextSide].\textrm{pop()}$
			\EndLoop
			\EndFunction
		\end{algorithmic}
	\end{algorithm}
	
	Given two starting vertices, we can formulate the problem of finding lane candidates as path pair enumeration. We enumerate pairs of simple paths, i.e. a sequence of vertices where no vertex occurs more than once.
	First, we define the solution set as follows:
	\begin{definition}[Solution set of path pairs]
		Let the set of all paths in the graph $\SearchGraph$ starting from vertex $s_l$ be $\SetLeftPaths$ and starting from $s_r$ be $\SetRightPaths$.
		The solution set $\SolutionSet$ for driving lane detection is the set of path pairs $\SolutionSet := \SetLeftPaths \times \SetRightPaths$. The feasible solution set 
		$\FeasibleSolutionSet$ is $\FeasibleSolutionSet := \{P \in \SolutionSet \mid \ConstratintDecider(P) \}$.
		\label{def:solution-set}
	\end{definition}

	The solution set $\SolutionSet$ is reduced further by imposing constraints using the predicate $\ConstratintDecider(\cdot)$ which returns true if the lane represented by the path pair satisfies all constraints and false otherwise.
	
	Algorithm \ref{alg:enumerate_path_pairs_rec} enumerates the feasible set $\FeasibleSolutionSet$ of path pairs by using a modified depth-first-search (DFS) that simultaneously searches both paths and incorporates heuristics as well as constraint satisfaction checking. It takes worst-case time $\mathcal{O}(|V|!)$  ($\mathcal{O}(5^{2|V|})$ with planar graphs) and $\mathcal{O}(|V| \cdot |E|)$ space.
	Despite the prohibitive worst-case runtime complexity, in practice the algorithm commonly is able to fully enumerate $\FeasibleSolutionSet$ due to the drastic reduction of the solution set imposed by the geometric constraints, most paths are not visited as the search backtracks early.
	
 	In general, we could also enumerate $\SetLeftPaths \times \SetRightPaths$ by enumerating $\SetLeftPaths$ and $\SetRightPaths$ separately (searching one path at a time) for which we would not need a modified DFS.
	However, as we mentioned earlier, our enumeration process is also a local optimization step where we prioritize finding the local optimum. If we think of the problem as an optimization over two dimensions ($\mathcal{L}$ and $\mathcal{R}$), enumerating only one set of paths at a time corresponds to a search over only one dimension at a time. Since the optimal solution could be anywhere in this 2D space, searching over only one dimension at a time for the optimal solution clearly is not an efficient optimization method. Instead, multivariate optimization methods typically use a search direction (e.g. the gradient vector), leading to an efficient search for local minima.
	The same idea applies to our algorithm -- by searching the two paths and thus over the two dimensions simultaneously, we can direct the search and find locally optimal solutions efficiently. The search direction is provided by a heuristic which indicates whether to continue searching further the left or the right path, it is named the \textit{Left-right-decider} and explained further in section \ref{sec:lr-decider}.
	
	Algorithm \ref{alg:enumerate_path_pairs_rec} begins by initializing an empty solution set (line 2) and then loops over all adjacent vertices (line 3). It obtains the last vertices of both paths (line 7) (which are labeled "current" as they are the current ones of DFS). It retrieves the already visited adjacent vertices of these current left and right vertices (line 8), as well as the unvisited ones which additionally are not already in the path (line 10). It then decides via a heuristic which unvisited adjacent vertices to visit next (line 14). The function returns when all have been visited for both paths (line 12). After choosing the next vertices, the \textit{side} variable $\NextSide$ is chosen which indicates whether to continue searching further the left or right path. We use 0 for left and 1 for right as a convention. To not exclude paths, the next side is chosen by the heuristic (line 16) only if both paths have unvisited adjacent vertices (line 15). If all adjacent vertices have been visited for one path and therefore only the other path has unvisited vertices, the path with the unvisited vertices is chosen as the next one to be searched (line 17). After adding the next vertex (line 18), the path pair is added to the solution set (line 21) if it satisfies the constraints (line 20). Then, the \textit{backtracking-decider} $\code{BTD}$ decides whether the constraints can not be satisfied anymore if they are violated (line 22). In such a case, the algorithm backtracks. Otherwise, the function is called recursively to continue the search (line 23). Note that the iteration counter $i$ has to be passed by reference unlike the other arguments to ensure the iterations are counted correctly.\\
	
	\subsection{Lanes for a closed course}
	Algorithm \ref{alg:enumerate_path_pairs_rec} does not enumerate closed driving lanes -- in order to be able to represent the racetrack as a closed course, a \textit{closing} is attempted as a post-processing step on each path pair found which closes both paths. A candidate can be closed if there are edges between the first and the last vertex of both the left and the right path. Also, the constraints must remain satisfied after closing. If such a closing is possible, the closed path pair is added to the solution set as an additional candidate.
	
	\subsection{Warmstarting}
	The enumeration of candidate solutions still suffers from combinatorial explosion limiting the search depth. In practice, this hinders detecting a complete racetrack which has an average length of 350m. The search depth achieved instead is in many cases below 100m (see section \ref{sec:enumeration-evaluation} for an evaluation).
	The solution to this problem is to reuse the most likely driving lane found in the last run of the algorithm for starting the enumeration, i.e. warmstarting. The warmstarting causes the enumeration algorithm to start with the most likely solution found (i.e. after running $\code{CLC}$ \ref{alg:clc} once). Here we use the fact that the map is initially not fully visible, the visible lane is commonly below these 100m which allows the algorithm to still detect the complete lane. Warmstarting causes the lane to be constructed incrementally over only a few runs of $\code{CLC}$ for many race track layouts.
	Overall, warmstarting thus enables a low-latency and real-time lane detection and at the same time completes the detected lane as the map is built. This allows for global race line optimization for minimizing the lap time \cite{min-time-opt}.\\

	\section{Search heuristics}
	
	In this section, we explain in detail how the enumeration algorithm \ref{alg:enumerate_path_pairs_rec} uses heuristics to guide the graph search, in particular we define the functions \code{LRD} and \code{NVD}.
	The two heuristics used ensure that during the limited number of iterations, locally optimal lane candidates are searched first. 

	\subsection{Next-vertex-decider} 
	\label{sec:nvd-def}
     \begin{figure}[!tbh]
         \begin{adjustbox}{width=.99\linewidth,center}
	       \includegraphics{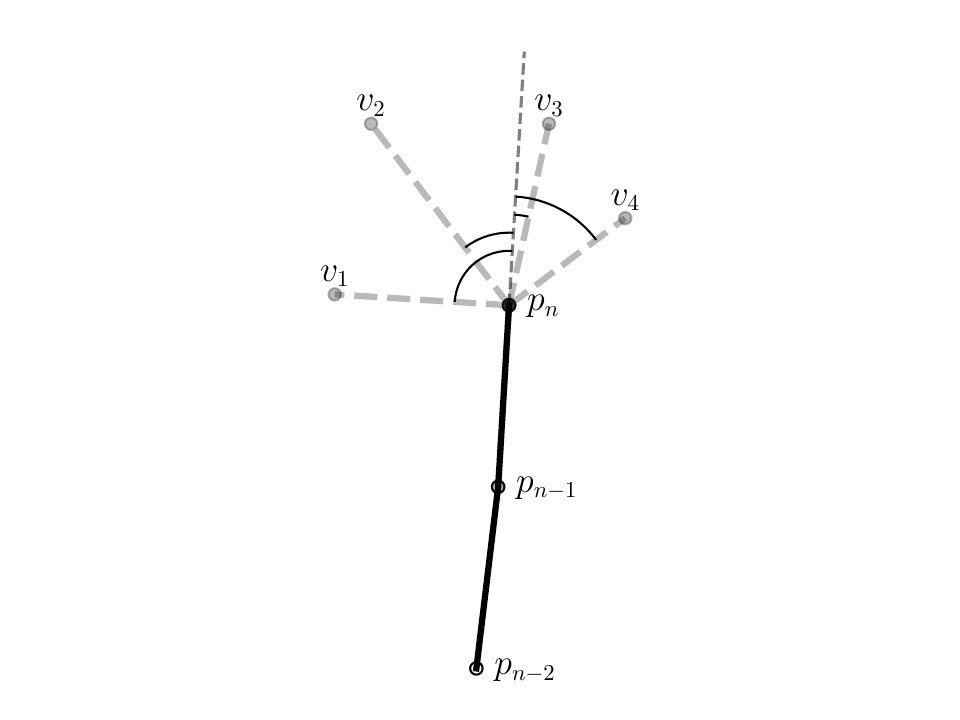}
	   \end{adjustbox}
	     \caption{Next-vertex decider: For the last point $p_n$ of either the left or right boundary we choose an adjacent vertex ($v_1, ..., v_4$) to visit next. In this example, $v_3$ is chosen as the next vertex adjacent to $p_n$ since the absolute angle between the line segments $\overline{p_{n-1} p_n}$ and  $\overline{p_n v_3}$ is the smallest from all segments to possible adjacent vertices.}
	     \label{fig:nvd}
	\end{figure}
	The next-vertex-decider selects the best adjacent vertex to visit next. It is always queried for the last vertex of both the left and right path.
	The idea of this heuristic is to predict straight lines based on the current direction, i.e. to extrapolate the boundary direction.
	For this, it considers angles of consecutive line segments of one boundary.
	The last vertex of a path corresponds to the last point $p_n$ of a polygonal chain $P = (p_1, ..., p_n)$ and has unvisited adjacent vertices $\mathcal{U}$. For every adjacent vertex $u \in \mathcal{U}$, we compute the absolute angle between the last segment $\overline{p_{n - 1}p_{n}}$ of the boundary and the segment $\overline{p_{n}p_{u}}$ from $p_n$ to $p_u$, the point corresponding to the unvisited vertex $u$. We then select the adjacent vertex which minimizes this angle (Fig. \ref{fig:nvd}).
	Consequently, the best next vertex for the path $P$ is chosen from the set of vertices $\mathcal{U}$ by the next-vertex decider \code{NVD} with:
	\begin{equation}
		\begin{aligned} 		 	
			\code{NVD}(P, \mathcal{U}) &= \argmin_{{u} \in \mathcal{U}}{ |\angle(\overline{p_{n - 1}p_{n}}, \overline{p_{n} p_{u}})|}
		\end{aligned}
		\label{eq:next_vertex}
	\end{equation}
	This is clearly a greedy heuristic which considers both boundaries independently and therefore may make suboptimal choices for the overall lane. 
	Despite the simplicity of the approach, it proves to be effective and is also computationally cheap. Note that we use memoization to increase the execution speed as this heuristic is usually queried multiple times for the same vertex.
	
	Initially, the polygonal chain is only a single point and therefore the segment  $\overline{p_{n - 1}p_{n}}$ is not defined. In this case we use the heading vector of the car instead of this segment. This is necessary as the search direction is ambiguous otherwise -- in case of a closed racetrack, the lane can either start in the car heading direction or in the direction opposite to it.
	
	\subsection{Left-right-decider}
	\label{sec:lr-decider}
	
	\begin{figure}[!tbh]
		\centering
		\begin{minipage}{0.49\linewidth}
			\begin{adjustbox}{width=\linewidth,center}
				\includegraphics{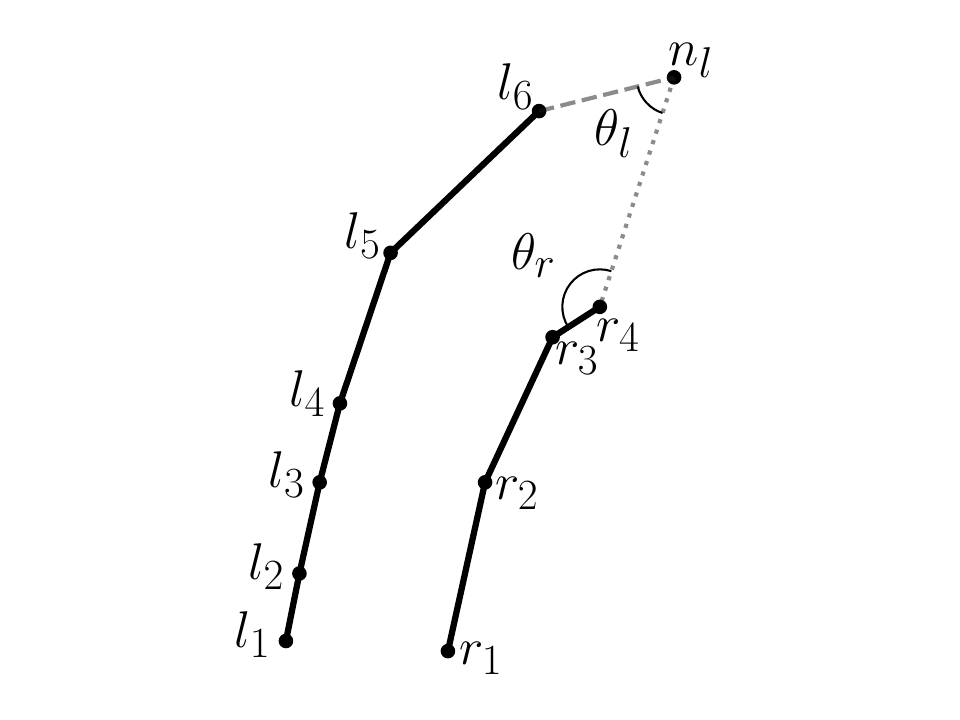}
			\end{adjustbox}
			\subcaption[second caption.]{Potential next vertex for the left path.}\label{fig:lr-decider-left-possible}
		\end{minipage}
		\begin{minipage}{0.49\linewidth}
			\begin{adjustbox}{width=\linewidth,center}
				\includegraphics{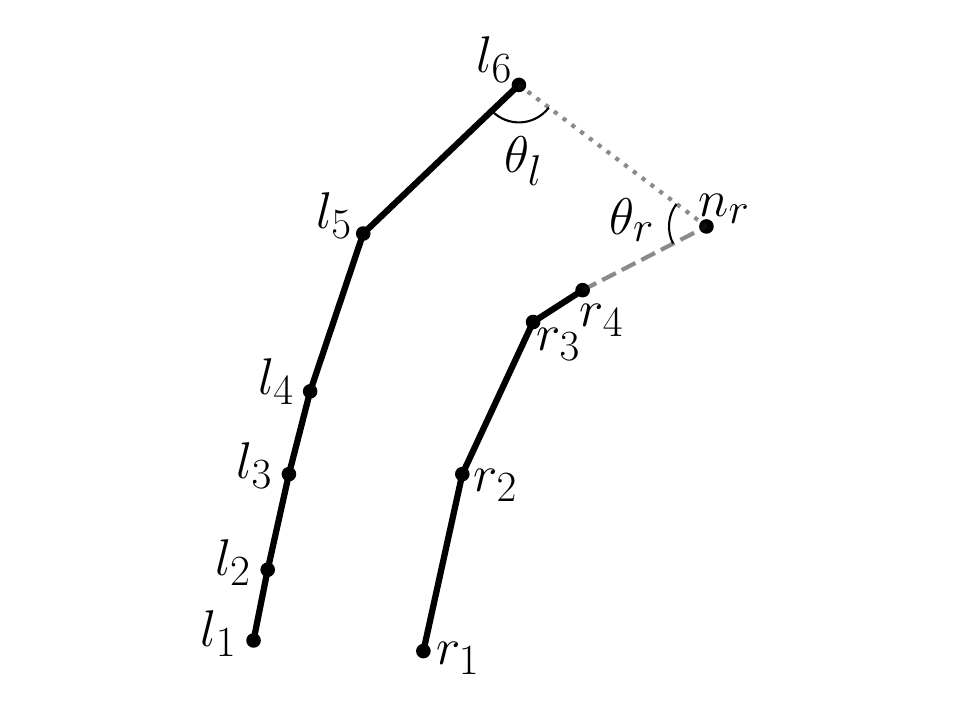}
			\end{adjustbox}
			\subcaption[first caption.]{Potential next vertex for the right path.}\label{fig:lr-decider-right-possible}
		\end{minipage}
		\caption{Left-right-decider: The left-right decider decides during the enumeration of path pairs whether a next vertex should be added to the left path (a) or to the right path (b). It is a geometric heuristic that tries to keep both paths equally \textit{advanced} by minimizing the deviation $|\theta_l - \theta_r|$. In this example, selecting the right vertex leads to a smaller deviation, therefore the left-right decider chooses the right side and the path pair becomes (b) in the next iteration.} 
		\label{fig:lr-decider-comparison}
	\end{figure}	
	
	The Left-right-decider $\code{LRD}$ is a heuristic that determines during the path pair enumeration whether it is best to continue searching the left or the right path.
	The heuristic is based on the geometric rule to try to maintain both polygonal chains (which correspond to the paths) equally \textit{advanced} (Fig. \ref{fig:lr-decider-comparison}). The previously described \NextVertexDecider-heuristic chooses two adjacent vertices $n_0$ (left) and $n_1$ (right), $\code{LRD}$ chooses then which one of them to add to the path pair\footnote{Note that we cannot simply add both vertices to the path pair at the same time, doing so would result in always enumerating path pairs of equal length which is only a subset of the solution set defined in Def. \ref{def:solution-set}.}. It outputs the side $\NextSide$ whether to add a vertex to the left or right path, the path on the other side remains unchanged. Therefore, if we would decide extending the left path and adding $n_0$, we obtain the path pair $P'_1 = ((l_1, ..., l_n, n_0), (r_1, ..., r_m))$, whereas if we decide for the right path and add $n_1$, we obtain the path pair $P'_2 = ((l_1, ..., l_n), (r_1, ..., r_m, n_1))$.
	$\code{LRD}$ is now simply a greedy heuristic which chooses the side that leads to the more likely one of these two path pairs.
	
	The heuristic works by comparing both potential path pairs $P'_1$ and $P'_2$ by computing two angles. It computes the two angles between the last line segments of both polygonal chains and the line segment connecting the last points of both polylines. The two angles $\theta_l$ and $\theta_r$ are defined given the left and right polygonal lines $(l_1, ..., l_n), (r_1, ..., r_m)$ as:

	\begin{equation}
		\begin{aligned}
			\theta_l =\angle(\overline{l_{n-1} l_n}, \overline{l_n r_m})\\
			\theta_r = \angle(\overline{r_{m-1} r_m}, \overline{r_m l_n})
		\end{aligned}
	\end{equation}

	The more equally advanced both polylines are, the more similar the two angles $\theta_l$ and $\theta_r$ are ($|\theta_r - \theta_l|$ is smaller).
	These two angles are computed for both lanes $P'_1$ and $P'_2$, for $P'_1$ they are $\theta_l^1, \theta_r^1$, for $P'_2$ they are $\theta_l^2, \theta_r^2$. 
	$\code{LRD}$ is a function which chooses the side corresponding to the lane with smaller $|\theta_r - \theta_l|$: 
	\begin{equation} 		 	
			\code{LRD}(P, n_0, n_1) = \begin{cases}
				0\textrm{ (left) } & |\theta_r^1 - \theta_l^1| < |\theta_r^2 - \theta_l^2| \\
				1\textrm{ (right) } & \, \textrm{otherwise} \\
			\end{cases}
		\label{eq:lrd}
	\end{equation}
 	For the special case in which initially both polygonal lines have zero line segments, an arbitrary side may be chosen.
 	Note that another special case occurs when the constraint decider may determine that the current lane does not satisfy the constraints and they can only be satisfied by extending a specific boundary (either left or right). In such a case the LR-decider simply chooses this side which can lead to constraint satisfaction. In such a case, the heuristic therefore tries to leave the unfeasible region as fast as possible.
	
	\section{Geometric lane constraints}
	In this section, we define the geometric constraints imposed on lanes as well as the two functions $\code{CD}(\cdot)$ and $\code{BTD}(\cdot)$ (see \ref{sec:backtracking-decider-btd}) used by algorithm \ref{alg:enumerate_path_pairs_rec}.
	
	We impose the following constraints:
	\begin{enumerate}
		\item Two points of the same boundary have a maximum distance $\MaxConeDistance$ ($\MaxConeDistance = 5.5m$)
		\item $\ConstraintSegments$: The absolute angle between two consecutive line segments of a boundary does not exceed a maximum $\MaxSemgentsAngle$ ($\MaxSemgentsAngle = 90\degree$)
		\item $\ConstraintSimplePoly$: The driving lane polygon is simple, i.e. does not intersect itself
		\item $\ConstraintLaneWidth$: The driving lane width is above $\MinLaneWidth$ and below $\MaxLaneWidth$ ($\MinLaneWidth = 2.5m$ and $\MaxLaneWidth = 6.5m$)
	\end{enumerate}
	
	The predicate indicating constraint satisfaction simply requires the satisfaction of all constraints:
	
	\begin{equation}     
		\label{eq:constraint-decider-def}         
		\begin{aligned}
			\ConstratintDecider(P) = \ConstraintSegments(P) \land \ConstraintLaneWidth(P) \land \ConstraintSimplePoly(P)
		\end{aligned}
	\end{equation}
	
	These constraints guarantee the output of a geometrically sound lane. One major consideration of these constraints is whether, once they are violated, they can be satisfied anymore by extending the lane. One insight into the efficiency of our lane detection algorithm is that in most cases, when a constraint is violated, it cannot be satisfied anymore. This allows for pruning the search space by backtracking.
	Note that, specific to the FSD competition, some of these constraints are explicitly mentioned in a rulebook, for example the maximum distance between two points, as well as the lane width constraint \cite[pp. 14-16]{fsgHandbook2023}\footnote{Note that, although the rulebook only states a minimum lane width, in practice we can also assume an upper bound.}.
	Recall that the first constraint is already satisfied by the graph construction (see \ref{sec:search-graph-construction}) and it therefore does not need to be checked by $\ConstratintDecider(\cdot)$. All other geometric constraints must be checked for every new path pair $P$. For this, we use efficient online algorithms that only perform computations that have not been done before.
	
	The constraints $\ConstraintSegments$ and $\ConstraintSimplePoly$ can be ensured as follows: For $\ConstraintSegments$, each time a new vertex and thus a new segment is added to the polygonal chain, we compute the absolute angle to the previous segment. If it is below $\MaxSemgentsAngle$ or there is no previous segment, the constraint $\ConstraintSegments$ is satisfied. For $\ConstraintSimplePoly$, we define the polygon of the driving lane by the sequence of points of the left boundary concatenated with the points of the right boundary in reversed order. To check whether it intersects itself, for each new segment we check for intersection with any other segment. If no segments intersect any other, the constraint is satisfied.
	
	In the following subsection, we will focus on the lane width constraint $\ConstraintLaneWidth$. It requires defining the width of the lane bounded by two polygonal chains, which is rather non-trivial and is also further complicated by the requirement of fast online-algorithms.

	\subsection{Driving lane width}
		
	In this subsection we outline the proposed approach for efficiently calculating the lane width bounded by two polygonal chains. 
	The general idea of the lane-width calculation is to establish a set of corresponding points which lie on the left and right boundary and then define the width as the Euclidean distance between them. The main task is therefore to establish these correspondences.
	
	The distance between two polygonal chains is a well studied topic in computational geometry with approaches such as the \textit{Fr{\'e}chet distance} \cite{Alt1995ComputingTF, locally-correct-frechet-matchings}. Another application is morphing between polygonal curves for which a \textit{width} between the curves has been considered explicitly in \cite{morphing-polylines-efrat}. All these approaches work by establishing corresponding points between both curves with reparametrization functions, named \textit{matchings}. The Fr{\'e}chet matching (leading to the Fr{\'e}chet distance) however is not unique and not every solution is suitable as a calculation of the width. A solution to the problem of uniqueness was proposed in \cite{locally-correct-frechet-matchings} where \textit{locally correct} Fr{\'e}chet matchings were considered, but the algorithm is difficult to implement and would likely be too slow for our application.
	
	We propose therefore a new algorithm for lane width calculation which meets the  following criterias: 
	
	\begin{enumerate}
		\item It is a natural way for measuring the lane width 
		\item Efficient online algorithms for computation can be derived
		\item It allows for an analysis whether the lane width constraint can be satisfied anymore once it is violated. 
	\end{enumerate}

	The third criterion states whether and when once the lane width is outside of bounds and thus the constraint is violated, it can be satisfied again by extending the lane. This is important as it allows the usage of backtracking when the constraint cannot be satisfied anymore and is thus essential for the efficiency of the graph search algorithm.
	
	We first make some preliminary definitions. 
	The length of a polygonal chain with $N$ points is defined as the sum of the length of its line segments:
	\begin{align}
		\label{boundary_len_eq}
		l_p = \sum_{i=1}^{N-1}||p_{i+1} - p_i ||_2
	\end{align}

	One way of describing a polygonal chain is as a curve $[0, l_p] \rightarrow \mathbb{R}^2$, mapping the \text{progress} along the curve to a 2D-point.
	We will use the slightly different parametrization for a polygonal chain where the domain depends on the number of points rather than the length of the curve: $P: [0, |P| - 1] \rightarrow \mathbb{R}^2$, where $|P|$ is the number of points of the polygonal chain. To define $P(x)$, we split $x$ in $x = i + \lambda$, the integer part $i \in \{0, 1, ..., |P| - 1\}$ which is simply the point index. The fractional part $\lambda \in [0, 1[$ states the progress along the $i$-th line segment. Given the sequence of points $(p_0, ..., p_{|P| - 1})$ of the polygonal chain, we define $P(i + \lambda) = (1 - \lambda)p_i + \lambda p_{i + 1}$. The set of points $P(x)$ for $x \in [i, i+1]$ denotes the i-th segment of $P$.
	
	We use this parametrization for both the left and right boundary, $\LeftSpline: [0, |\LeftSpline| - 1] \rightarrow \mathbb{R}^2$ and $\RightSpline: [0, |\RightSpline| - 1] \rightarrow \mathbb{R}^2$.
	We define the sequence of \textit{matching points} as $M_{ps} = {\{(u_i, v_i) \in [0, |L| - 1 ] \times [0, |R| - 1] \}_{i=0}^{|M|-1}}$, consisting of pairs of parameters for the left and right polygonal chains \cite{locally-correct-frechet-matchings}. 
	The sequence ${\{ (L(u_i), R(v_i)) \mid (u_i, v_i) \in M_{ps}(L, R)  \}_{i=0}^{|M_{ps}|-1}}$ is a sequence of endpoints of line segments which we call the \textit{matching lines}, the length of these segments is the driving lane width (see Fig. \ref{fig:lane-width-calc}).
	
	\subsubsection{Establishing a matching}
	The idea for an efficient algorithm calculating these matching points is to use a nearest-neighbor search (NN-search) for closest line segments between the left and right polygonal chain. To ensure that every vertex is matched as well, we also perform nearest neighbor search with vertices as a query. We use analytic formulas for calculating the shortest line segment between two line segments.
	
	First, we define the search space $\Omega_{i}(k,s)$ for arguments $a$ and $b$ to both polygonal chains. It depends on whether we use the $i$-th vertex ($k=0$) or the $i$-it segment ($k=1$) as a query and whether
	we use the left polygonal chain as a query with the right one as target ($s=0$) or vice-versa ($s=1$):
	
	\begin{equation}
		\begin{aligned}
			\Omega_{i}(k,s) &= 
				\begin{cases}
					\{i\} \times [0, |\RightSpline| - 1] \textrm{ if } s = 0, k = 0 \\
					[i, i+1] \times [0, |\RightSpline| - 1] \textrm{ if } s = 0, k = 1 \\
					[0, |\LeftSpline| - 1] \times \{i\} \textrm{ if } s = 1, k = 0 \\
					[0, |\LeftSpline| - 1] \times [i, i+1] \textrm{ if } s = 1, k = 1 \\
				\end{cases}
		\end{aligned}
	\end{equation}
	
	This search space can be interpreted as follows: The nearest-neighbor search always has as a target domain the complete polygonal chain, i.e. the intervals $[0, |\LeftSpline| - 1]$ and $[0, |\RightSpline| - 1]$. The query for the NN-search can either be a single line segment of the polygonal chain ($[i, i+1]$), or a single point, i.e. $\{i\}$, for $i$ being an integer.
	
	We define then the nearest-neighbor search as yielding the matching points $M(\LeftSpline, \RightSpline, s, k)$ between $\LeftSpline$ and $\RightSpline$:
	\begin{equation}
		\begin{aligned}
			M(\LeftSpline, \RightSpline, k, s) &=\\ \biggl\{(u_i, v_i) &= \argmin_{ (a, b) \in \Omega_{i}(k,s)} || \LeftSpline(a) - \RightSpline(b) ||_2 \biggl\}_{i=0}^{N_m(k, s) - 1}\\
			\label{eq:lane-width-nn-search}
		\end{aligned}
	\end{equation}
	
	$N_m(k, s)$ is the number of matching points, it is defined as:
	\begin{equation}
		\begin{aligned}
			N_m(k, s) &= 
			\begin{cases}
				|\LeftSpline| - k \textrm{ if } s = 0\\
				|\RightSpline| - k \textrm{ if } s = 1\\
			\end{cases}
		\end{aligned}
	\end{equation}
	
	When using segments as a query ($k=1$) instead of vertices ($k=0$), there is one matching point less due to a polygonal chain having one segment less than it has points, this is why we subtract $k$. 
	Also, depending on whether the polyline $\LeftSpline$ or $\RightSpline$ is the query, the number of matching points depends either on $|\LeftSpline|$ or on $|\RightSpline|$. 
	
	The proposed method for computing the matching points $\MatchinPoints$ between the two boundaries $\LeftSpline$ and $\RightSpline$ is by combining nearest points-to-segments and segments-to-segments and using $\LeftSpline$ and $\RightSpline$ as query as well as a target:
	\begin{equation}
		\begin{aligned}
			\MatchinPoints(\LeftSpline, \RightSpline) &= \bigcup_{s, k \in \{0, 1\}^2}{M(\LeftSpline, \RightSpline, s, k)}
		\end{aligned}
		\label{eq:matching-points}
	\end{equation}
	
	The lane width $\LaneWidth$ between the two boundaries $\LeftSpline, \RightSpline$ is defined simply as the Euclidean distance between corresponding points:
	\begin{equation}
		\begin{aligned}
			\LaneWidth(\LeftSpline, \RightSpline) = \{ || \LeftSpline(u_i) - \RightSpline(v_i) ||_2, (u_i, v_i) \in \MatchinPoints(\LeftSpline, \RightSpline) 	\}
		\end{aligned}
		\label{eq:lane-width}
	\end{equation}
	
	The lane width constraint predicate $\ConstraintLaneWidth{}(P)$ requires that the lane width is everywhere above $\MaxLaneWidth$ and below $\MinLaneWidth$:
	
	\begin{equation}
		\label{eq:lane-width-constraint-def}
		\begin{aligned}
		\ConstraintLaneWidth(P) = \bigwedge_{ w_i \in \LaneWidth(L, R)} \MinLaneWidth < w_i < \MaxLaneWidth 
		\end{aligned}
	\end{equation}
	
	\subsubsection{Online-algorithm for lane width calculation}
	\begin{figure}[!h]
		\begin{adjustbox}{width=.49\linewidth}
			\includegraphics{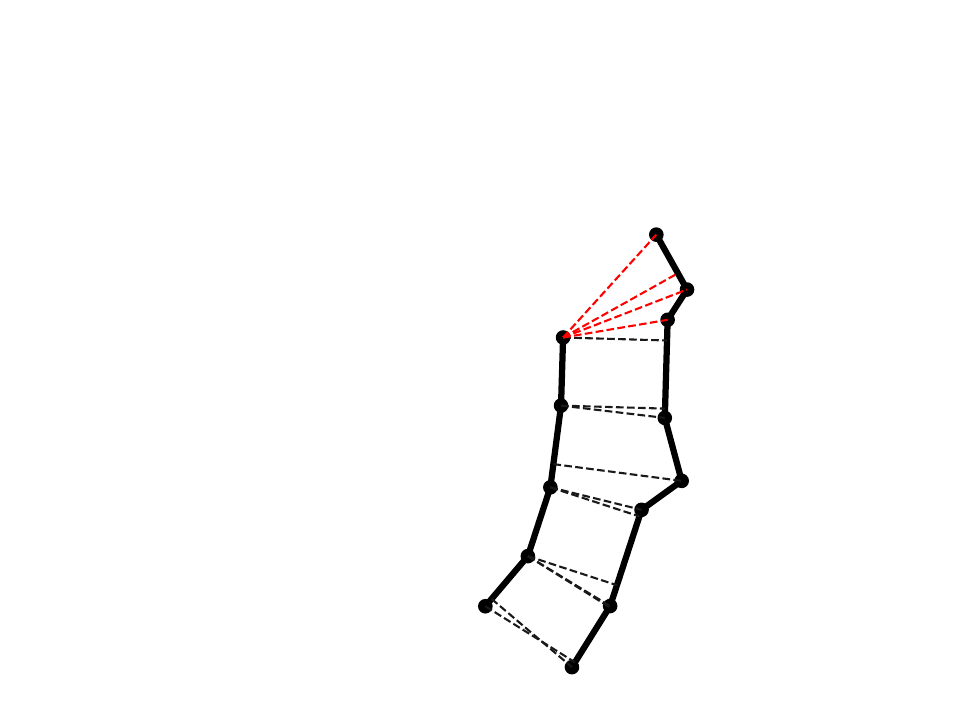}
		\end{adjustbox}
		\begin{adjustbox}{width=.49\linewidth}
			\includegraphics{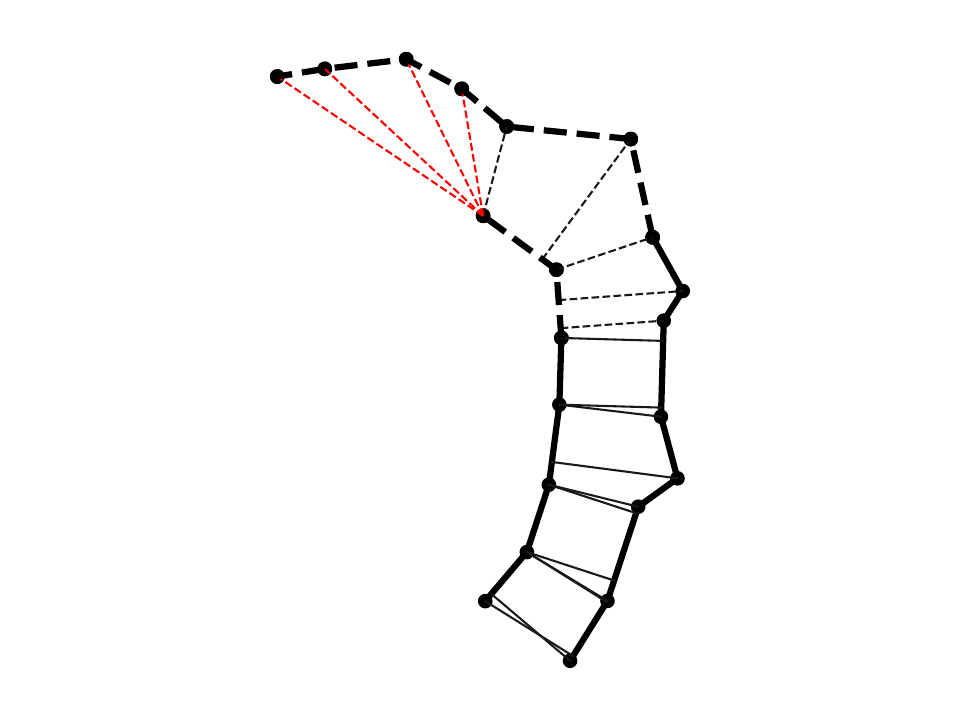}
		\end{adjustbox}
		\caption{Left: To obtain the driving lane width, \textit{matching lines} (thin dashed lines) are computed between the two lane boundaries (solid lines). The lengths of these matching lines define the driving lane width. The online algorithm may recompute all the \textit{mutable} matching lines (red) which may overestimate the lane width. Right: After extending the boundaries, new matching lines are computed (dashed), replacing the mutable ones. The set of \textit{fixed} matching lines (black dashed and solid) is only extended, the fixed matching lines from the previous iteration (black solid) are never recomputed.} 
		\label{fig:lane-width-calc}
	\end{figure}
	\begin{algorithm}
		\caption{Online algorithm for lane width calculation}
		\label{alg:online-lane-width}
		\begin{algorithmic}[1]
			\Require $\LeftSpline, \RightSpline$: two lane boundaries (polygonal chains), 
			$M_{fixed}^{t-1}$: fixed matching points from last iteration, sorted
			\Ensure $M_{fixed}^{t}, M_{mut}^{t}$: Updated fixed and mutable matching points
			\Function{OnlineLW}{$\LeftSpline, \RightSpline, M_{fixed}^{t-1}$}
			\If{$M_{fixed}^{t-1} = \EmptySet$}
			\State $u_s, v_s \leftarrow (0, 0)$
			\Else 
			\State $u_s, v_s \leftarrow M_{fixed}^{t-1}.\textrm{back}()$  
			\EndIf
			\State $M' \leftarrow \textrm{compute matching for } L, R$ starting  from $ u_s , v_s $ (Eq. \ref{eq:matching-points})
			\State sort $M'$ and split into $M_{fixed}'$ and $M_{mut}^{t}$
			\State $M_{fixed}^t \leftarrow M_{fixed}^{t-1} \cup M_{fixed}'$
			\State \Return{$M_{fixed}^{t}, M_{mut}^{t}$}
			\EndFunction
		\end{algorithmic}
	\end{algorithm}
	
	Instead of applying formula \ref{eq:matching-points} every time for every new pair of boundaries, we propose the following online-algorithm (Alg. \ref{alg:online-lane-width}). The idea is to divide the set of matching points into two sets, the set of \textit{fixed} matching points which are never recomputed but instead only new ones may be added. The second set of \textit{mutable} matching points may be recomputed each time the boundaries are extended (Fig. \ref{fig:lane-width-calc}). This split into two sets not only allows for greater efficiency but also for backtracking. If the lane width does not satisfy constraints and the violating matching point belongs to the set of fixed matching points, we can backtrack since this matching will not be recomputed. See lemma \ref{lemma:constraint-lw-bt-criterion} for a detailed analysis of the backtracking criteria.
	
	We now describe in detail Algorithm \ref{alg:online-lane-width}.
	The computation of matching points starts from the last matching point from the previous iteration (line 5), or the beginning of the boundaries in case it is the first iteration (line 3).
	After performing the nearest-neighbor search for matching points (line 6), the set of computed matching points is split into the fixed and mutable ones based on the first matching point which matches at least one end of either boundary (i.e it is either $(|L| - 1, v)$ or $(u, |R| - 1)$). All matching points which come before this matching point are fixed, all which come after it are mutable.
	For this, an order of matching points needs to be established. 
	Since the search in \ref{eq:matching-points} does not guarantee the matching points to be monotone, i.e. $(u_i \leq w_i \land v_i \leq x_i) \lor (u_i \geq w_i \land v_i \geq x_i) \, \forall ((u_i, v_i), (w_i, x_i)) \in \MatchinPoints^2$ generally does not hold, we establish a lexicographic order. 
	
	After splitting the computed matching points and sorting, the sequence of new fixed matching points $M_{fixed}'$ (which may be empty) is appended to the ones from the previous iteration  $M_{fixed}^{t-1}$ (line 8), which maintains the order. The constraint Eq. \ref{eq:lane-width-constraint-def} is evaluated over both sets of matching points, i.e. $M_{fixed}^{t}$ and $M_{mut}^{t}$.
	
	\subsection{Backtracking-decider}	
	\label{sec:backtracking-decider-btd}

	Note that, most of the time, the geometric constraints can no longer be satisfied once they have been violated, i.e. the graph search can immediately backtrack. However, a backtracking-decider $\code{BTD}(\cdot)$ is still required because there are some exceptions to this rule. It returns true if we can backtrack and false it not. It is defined as follows. $\code{BTD}(\cdot)$ generally returns false if no constraints are violated, i.e. ~$\lnot \ConstratintDecider(P) \implies \lnot \BacktrackingDecider(P, v_l, v_r) \, \forall P \in \SolutionSet, v_l, v_r \in \{\textrm{true}, \textrm{false}\}^2$. In all other cases it returns true except in the ones where the constraints $\ConstraintSimplePoly$ and $\ConstraintLaneWidth$ still can be satisfied (see lemmata \ref{lemma:polygon-simplicity-bt-criterion} and \ref{lemma:constraint-lw-bt-criterion}), in which it returns false.

	\section{Ranking of candidate driving lanes}
	In this section we outline the second essential step of our algorithm $\code{CLC}$ \ref{alg:clc} after candidate solutions have been enumerated, which is choosing the most likely candidate via ranking. This task is challenging as the only measurements indicating the likelihood of a lane are the boundary points and there can be ambiguities. One simple but effective heuristic which works well in case of no false-positives is to simply select the lane with the greatest length.
	Other criteria are the curvatures of the two boundaries (which should not exceed a maximum) and the similarity between the two boundaries (which should be high). These geometric criteria are however subtle and tuning handcrafted heuristics exploiting them can be difficult.
	
	We propose therefore for this pattern recognition problem the usage of a supervised machine learning approach. Instead of assigning numerical likelihood values to each candidate lane, we use the formulation as a ranking-problem where a binary classifier chooses one candidate over another (pair-wise ranking) as more likely. This is assumed to improve generalizability over directly predicting a likelihood in case the lane length varies a lot.
	
	The input to the neural network is an 8-dimensional feature vector derived from the driving lane. We use as features (1) the lane length (computed as the mean of the length of the two boundaries) and the number of points of the left (2) and right (3) boundary. As an estimation of the geometric implausibility, we compute the variance of the lane width (4). Also, we compute the variance of line segment lengths and the angles between consecutive line segments for both boundaries (5, 6 and 7, 8 respectively). We also tried adding the respective mean values, but no improvement was observed. These are clearly hand-crafted features, but parameter-free and thus removing overall the need for hand-tuning parameters. 
	
	The model used is a 2-layer fully-connected neural network inspired by \textit{RankNet} \cite{Burges2005LearningTR} which is a lightweight architecture suitable for real-time inference. 
	It has in the first layer $8 \cdot 100$ parameters and in the second layer $100 \cdot 1$ parameters, using ReLU as activation function.
	Similar to the RankNet-approach, the predicted output probability $p_{pred}$ is computed from the two feature vectors $x_1$ and $x_2$ with $p_{pred} = \textrm{sigmoid}(\textrm{Net}(x_1) - \textrm{Net}(x_2))$ where $\textrm{Net}(\cdot)$ is the two-layer neural network. Binary-cross-entropy-loss is used as a loss function, for which the ground-truth probability $p_{GT}$ is derived from the intersection over union (IoU)-values between a candidate lane polygon and the ground-truth lane polygon. For two candidates, this yields the two values  $\textrm{IoU}_{1}, \textrm{IoU}_{2}$, $p_{GT}$ is defined then as:
	\begin{align}
		p_{GT} = \textrm{sigmoid}(\lambda_{\textrm{iou}} (\textrm{IoU}_{1} - \textrm{IoU}_{2}))
		\label{eq:gt-prob-pairwise-ranking}
	\end{align}
	The factor $\lambda_{\textrm{iou}}$ is required to increase the IoU-difference as otherwise training happens very slowly. It is a non-critical parameter, in our experiments values from $\lambda_{\textrm{iou}}=10$ to $\lambda_{\textrm{iou}}=150$ worked equally well.
	The dataset used for training consists overall of lists of enumerated candidates each time the lane detection runs. The lane detection runs continuously as the map of the environment is build and therefore
	has to work reliably over all partial (potentially noisy) maps.
	To include only significant map changes, we use as a criterion the traveled distance of the vehicle which has to be above 1m. To construct the training dataset for the ranking problem, we enumerate candidates for every such partial map which yields candidate lists.
	We then pair candidates and compute the ground-truth probability with Eq. \ref{eq:gt-prob-pairwise-ranking} to obtain the dataset for the binary classification problem. 

	\section{Evaluation}
	\begin{figure*}
		\begin{adjustbox}{width=.99\linewidth,center}
			\includegraphics{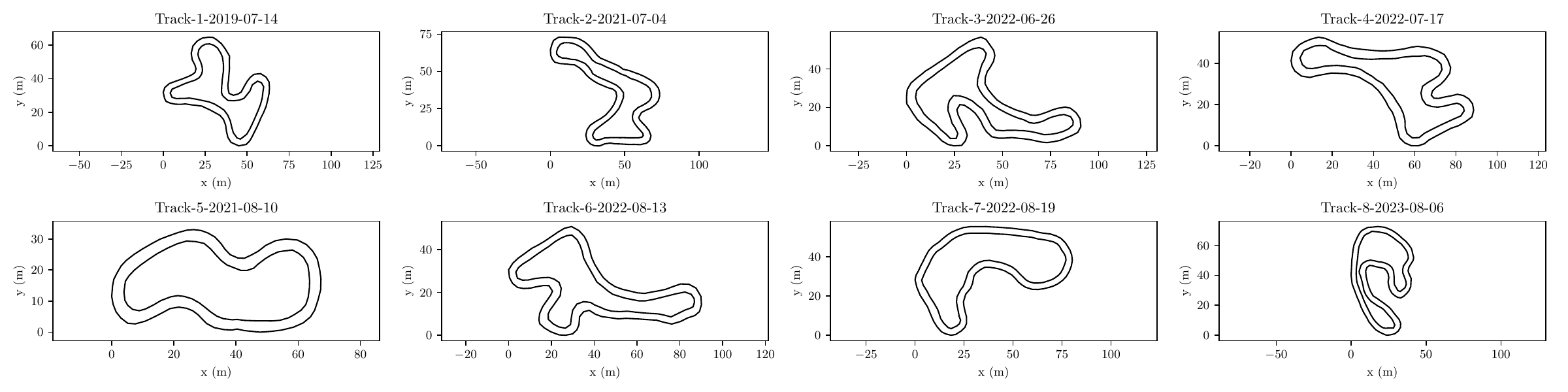}
		\end{adjustbox}
		\caption{The dataset used for evaluation contains realistic race track layouts including curvy ones which are challenging for the lane detection.}
		\label{fig:datasettrackssummary}
	\end{figure*}
	
	This section evaluates the proposed lane detection algorithm regarding its accuracy, efficiency and robustness against false positives. It shows the ability to detect the entire visible lane in most instances in real time. We use a dataset consisting of several race tracks built from real data with different realistic track layouts.
	The proposed algorithm is implemented in C++, geometric computations were performed in single-precision floating point arithmetic (instead of arbitrary-precision arithmetic) except the Delaunay-triangulation which is computed by the CGAL-library \cite{cgal:hs-chdt3-13b}. The neural network was implemented with PyTorch for training, whereas TensorFlow Lite was used for inference and evaluation, performing the inference on CPU. 
	All experiments were executed on a PC with an Intel i7-12700K CPU with 32 GiB of RAM, comparable to the PC installed in our formula student race car. 
	
	\subsection{Dataset}
	
	As no open datasets are available which contain realistic race track layouts with boundaries marked by points, we created our own dataset for evaluation built from real recorded data on various race track layouts (Fig. \ref{fig:datasettrackssummary}). It contains for every track layout a map with the 2D-points and the trajectory driven by the car. The points were mapped by the LiDAR-based SLAM-algorithm  \textit{Rao-blackwellized particle filter-SLAM} (RBPF-SLAM) \cite{4543250}, a variant of FastSLAM \cite{MontemerloTKW02} that is implemented in the Mobile Robot Programming Toolkit (MRPT) library \cite{MRPT-home}. For every map, the ground-truth driving lane boundaries were annotated manually. Each time the SLAM updated the map, it also updated the car pose. These car poses are used as a starting parameter for the algorithm and are part of the dataset as well. Overall, the dataset consists of \DatasetNumTracks{} tracks and a total of 1964 car poses. 
	
	\subsection{Experiments and Metrics} 

	We evaluate the lane detection algorithm in the same way as it is used on the racecar, every time the map is updated. To avoid the tedious task of manually annotating partial maps, we instead annotate the complete map once and then simulate partial maps using this complete map. For this, we use the known perceptual field and the car pose, and filter out the points outside the perceptual field. We then simulate false-positives by randomly sampling from an uniform distribution in the field of perception. The LiDAR-sensor of the car has a field of perception of a half-circle in front of the car with a radius corresponding to the perception range. The noise-free maps were augmented with false positive rates of 10\%, 30\%, and 50\%, and two different ranges of perception were simulated, 30m and 50m, resulting in a total of 6 variations of the dataset. 
	
	The neural network was trained on enumerated candidates which yielded approx. 270.000 candidates, grouped into 1964 lists. The training schedule was using 200 epochs with a batch-size of 8192 and Adam as optimizer with an learning-rate of $0.008$. As the overall architecture has only 1001 parameters, it takes only around two minutes to train on an Nvidia RTX 3070 Ti GPU.
	
	As metrics, we use the Intersection-over-Union (IoU) between the ground-truth lane and the predicted lane. Also, we determine for each predicted lane whether it diverges from the ground-truth lane, and if so, the distance after which it does, we name it the \textit{divergence distance}. We also compare the length of the predicted lane to the length of the ground-truth lane to identify cases where a shorter lane is predicted. If every vertex of both boundaries is predicted correctly, such a lane prediction is said to be ground-truth, i.e. exactly matching the correct lane.
	Additionally to stating the values of these metrics directly, we introduce categories for predictions based on these metrics for easier interpretability of the results (Fig. \ref{fig:failure-cases-categories}).
	A non-ground-truth prediction can either be too short (Fig. \ref{fig:failure-case-too-short}), diverging (Fig. \ref{fig:failure-case-diverging}) (which may be a critical failure case if it happens below 20m) or near-ground truth (Fig. \ref{fig:failure-case-almost-gt}).
	
	\subsection{Used parameters} 
	
	The maximum number of iterations $it_{max}$ was set in all experiments to $\DefaultMaxIterations$ except for Fig. \ref{fig:bestrankedneedediterations}, where it was set to $20000$. The values for the geometric constraints were set to $\MinLaneWidth = 2.5$ and $\MaxLaneWidth = 6.5$, the maximum distance between two points to $\MaxConeDistance=5.5$. The maximum angle between two consecutive segments was set to $\MaxSemgentsAngle = 90\degree$. These values were set such that no ground-truth lane occurring in the dataset was detected as violating the constraints, this was verified experimentally.
	\subsection{Results}

	\begin{figure*}[ht]
		\centering
		\begin{minipage}{0.3\textwidth}
			\begin{adjustbox}{width=\linewidth,center}
				\includegraphics{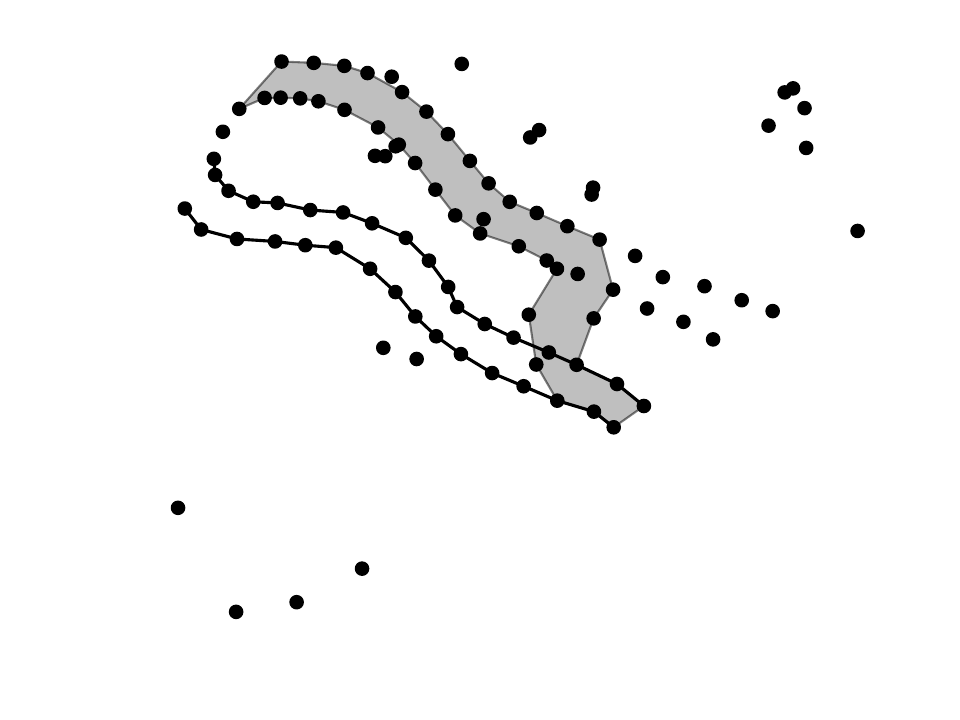}
			\end{adjustbox}
			\subcaption[second caption.]{Failure case: Predicted lane diverges (leaves the track), critical if below 20m}\label{fig:failure-case-diverging}
		\end{minipage}
		\begin{minipage}{0.3\textwidth}
			\begin{adjustbox}{width=\linewidth,center}
				\includegraphics{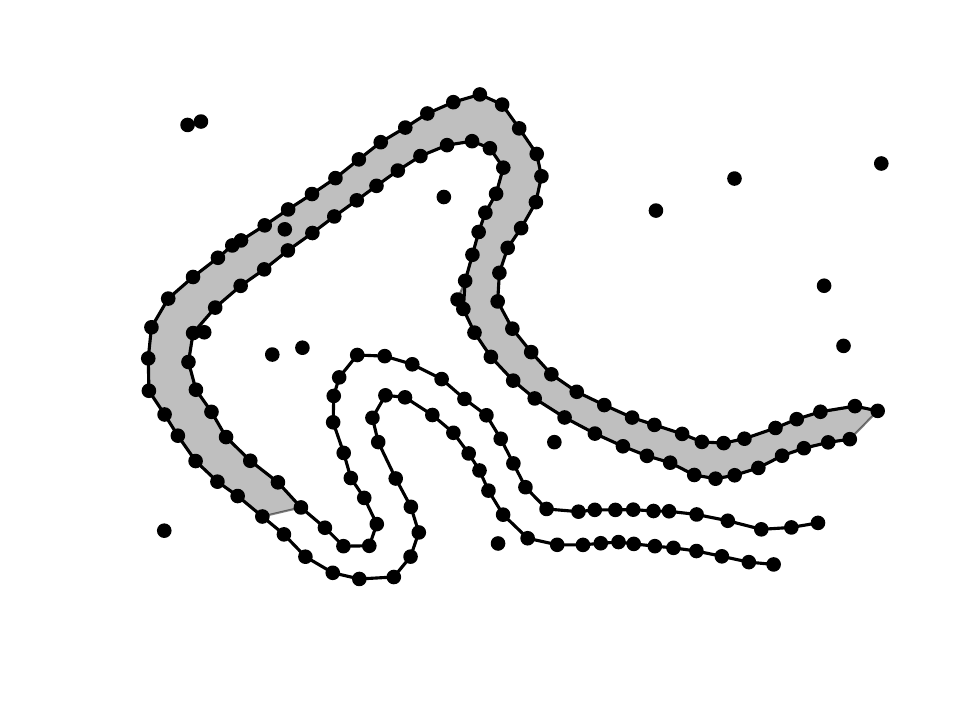}
			\end{adjustbox}
			\subcaption[first caption.]{Non-critical failure case: predicted lane is too short, (less than 90\% of the ground-truth length)}\label{fig:failure-case-too-short}
		\end{minipage}
		\begin{minipage}{0.3\textwidth}
			\begin{adjustbox}{width=\linewidth,center}
				\includegraphics{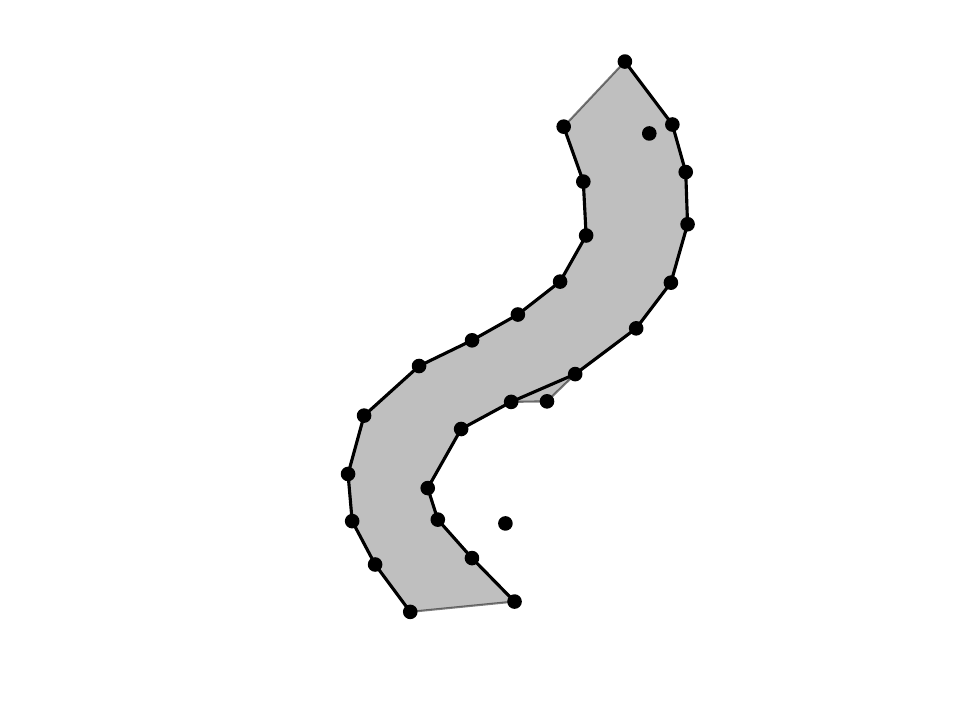}
			\end{adjustbox}
			\subcaption[first caption.]{Success case: Near ground-truth}\label{fig:failure-case-almost-gt}
		\end{minipage}%
		\caption{We classify each prediction into the following categories: Besides ground-truth predictions, the predicted lane may diverge \ref{fig:failure-case-diverging}. It is critical in case it happens below 20m, as it can cause the car to leave the track. The lane can be too short \ref{fig:failure-case-too-short}, this may only make the racing-line suboptimal. If the predicted lane is neither of these, it is a \textit{near ground-truth} solution \ref{fig:failure-case-almost-gt}. Such predictions typically have a high overlap with the ground-truth lane. (The gray polygon is the prediction, the black boundaries are ground-truth.)} 
		\label{fig:failure-cases-categories}
	\end{figure*}

	\begin{table*}[ht]
		\centering
		\input{summary-table.tex}
		\caption{Analysis of the lane prediction accuracy: Rates of different prediction categories depending on the false-positive rate. With zero false-positives, the ground-truth predictions are dominant, but they decrease rapidly as the number of false-positives increase. Critical failures where the driving lane diverges below 20m only occur in about 1\% of instances, except at extreme FP-rates of 50\%. Predictions that are too short rarely occur overall.}
		\label{tab:results}
	\end{table*}

	\begin{figure}[htb]
		\begin{adjustbox}{width=.7\linewidth,center}
			\includegraphics{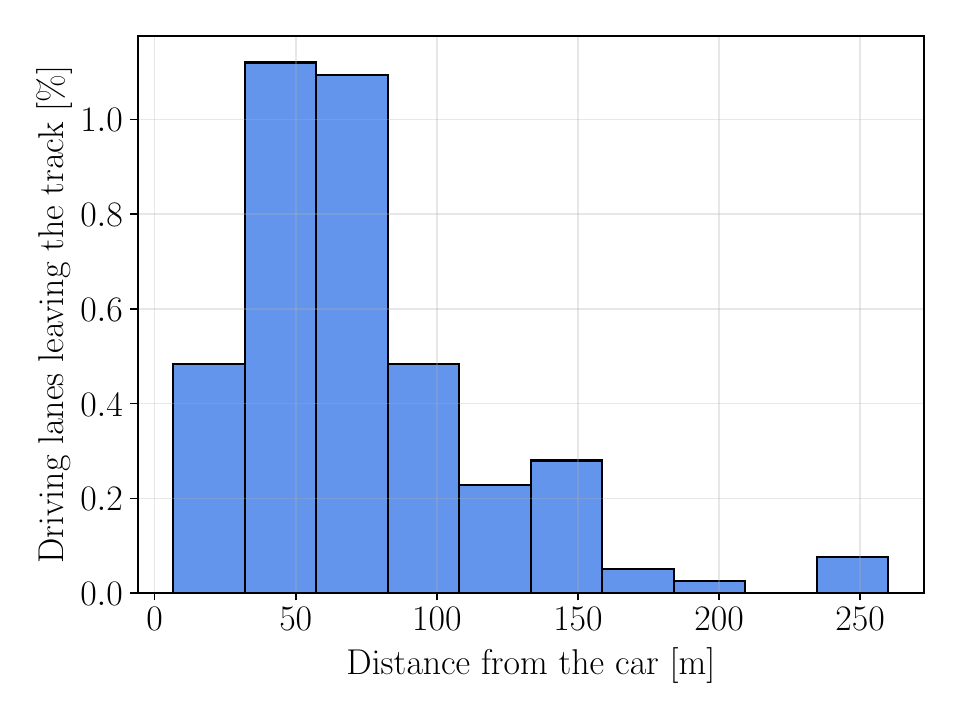}
		\end{adjustbox}
		\caption{Distance till the divergence point in case of divergent lane predictions. The overall rate of divergent predictions is low. (Shown for FP-rates of 0\% to 30\%)}
		\label{fig:drivinglanelengths-leaving-gt}
	\end{figure}

	\begin{figure*}[htb]
		\begin{adjustbox}{width=0.12\linewidth}
			\includegraphics{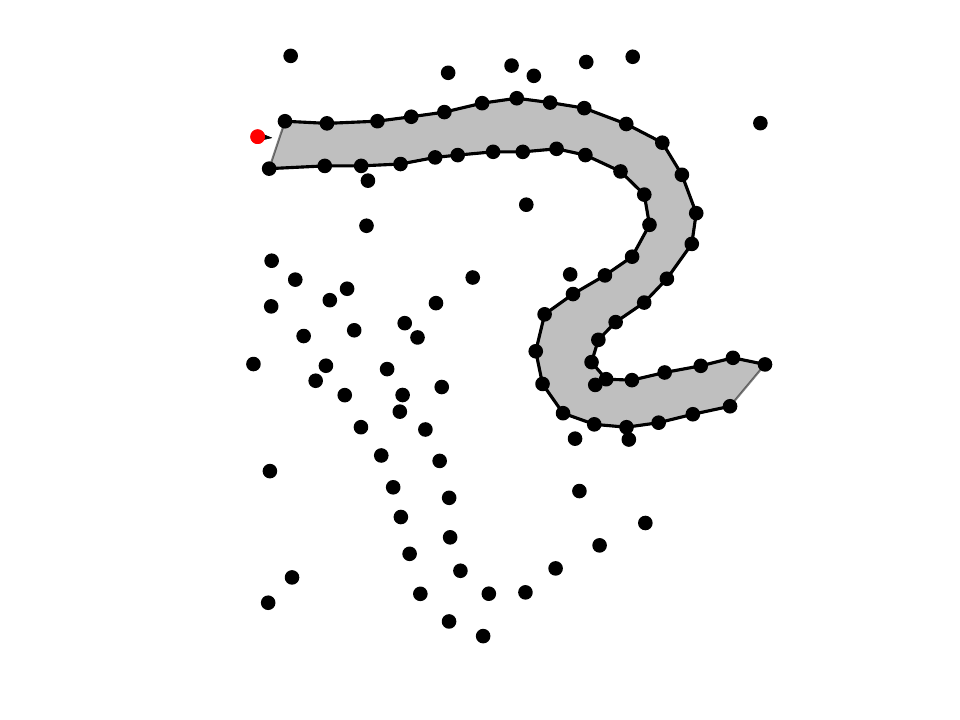}
		\end{adjustbox}
		\begin{adjustbox}{width=0.12\linewidth}
			\includegraphics{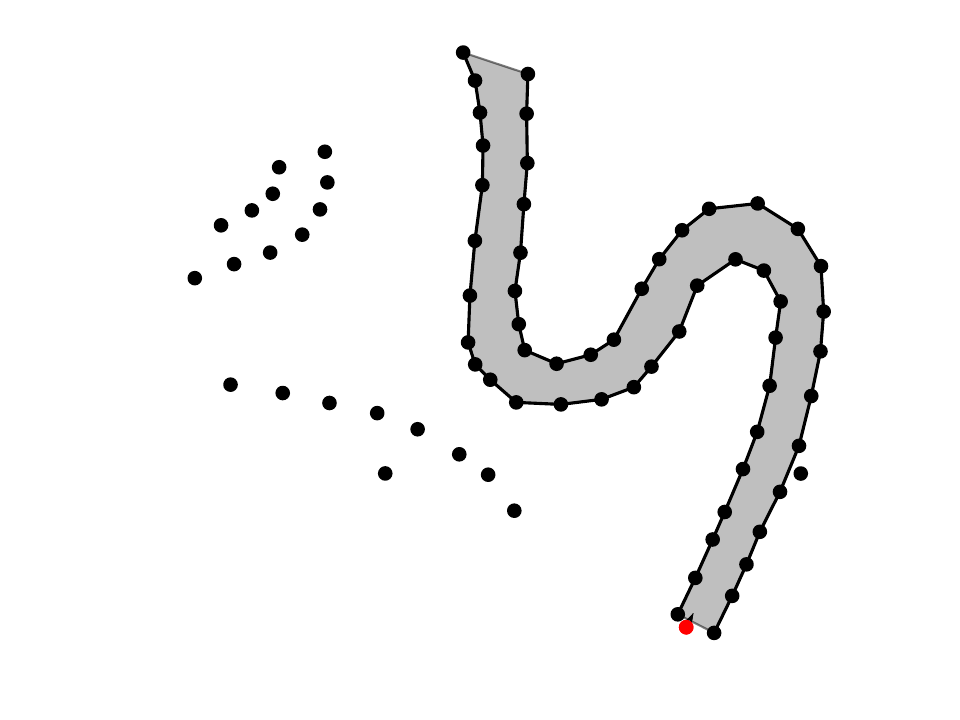}
		\end{adjustbox}
		\begin{adjustbox}{width=0.12\linewidth}
			\includegraphics{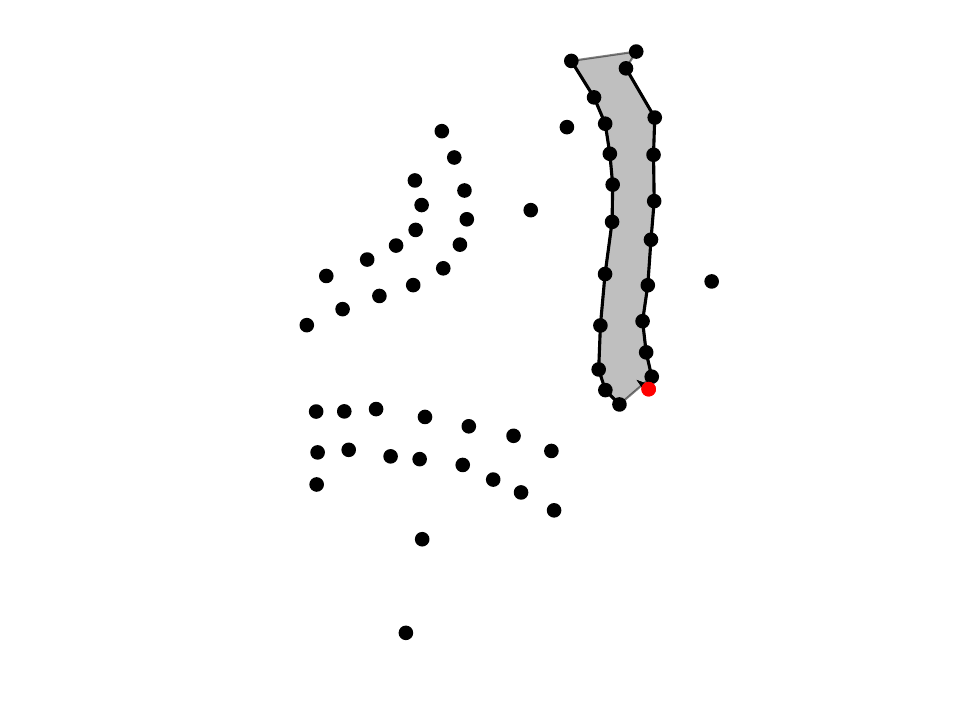}
		\end{adjustbox}
		\begin{adjustbox}{width=0.12\linewidth}
			\includegraphics{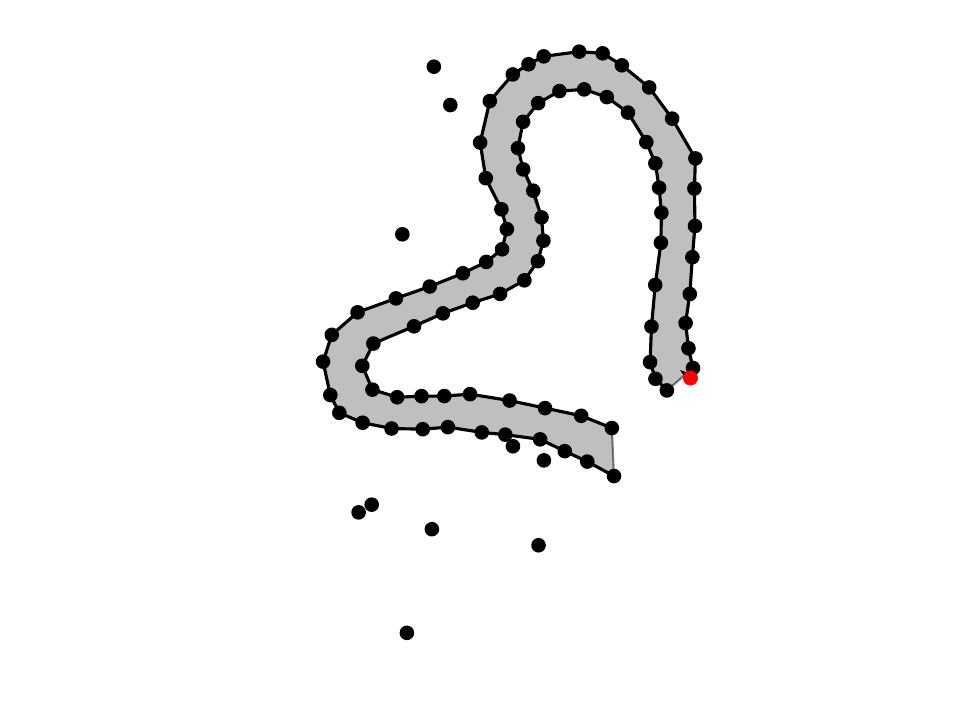}
		\end{adjustbox}
		\begin{adjustbox}{width=0.12\linewidth}
			\includegraphics{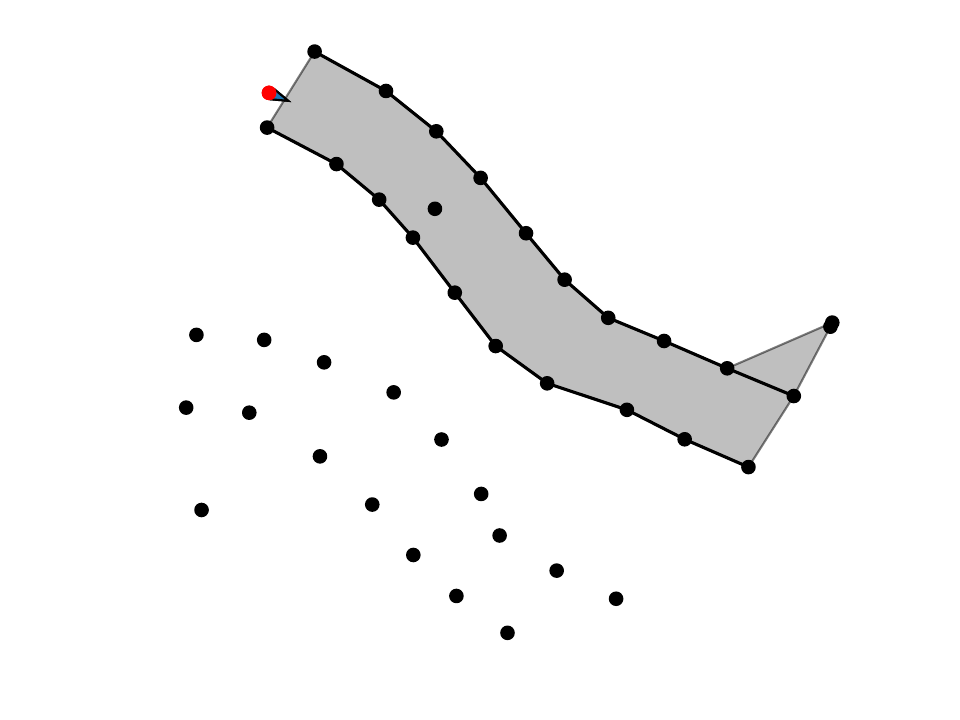}
		\end{adjustbox}
		\begin{adjustbox}{width=0.12\linewidth}
			\includegraphics{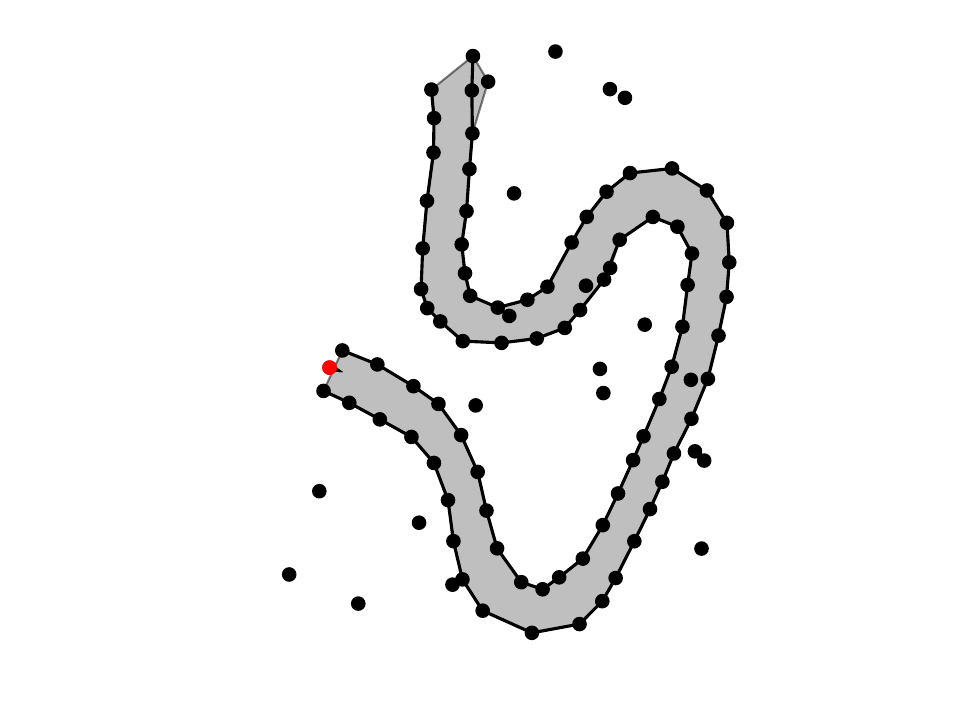}
		\end{adjustbox}
		\begin{adjustbox}{width=0.12\linewidth}
			\includegraphics{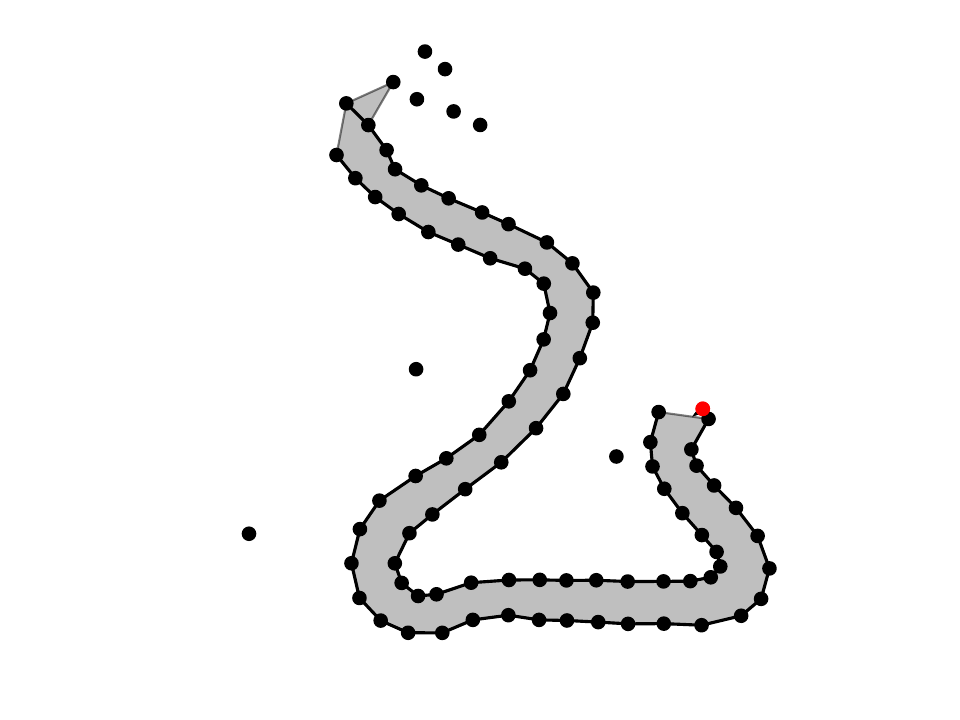}
		\end{adjustbox}
		\begin{adjustbox}{width=0.12\linewidth}
			\includegraphics{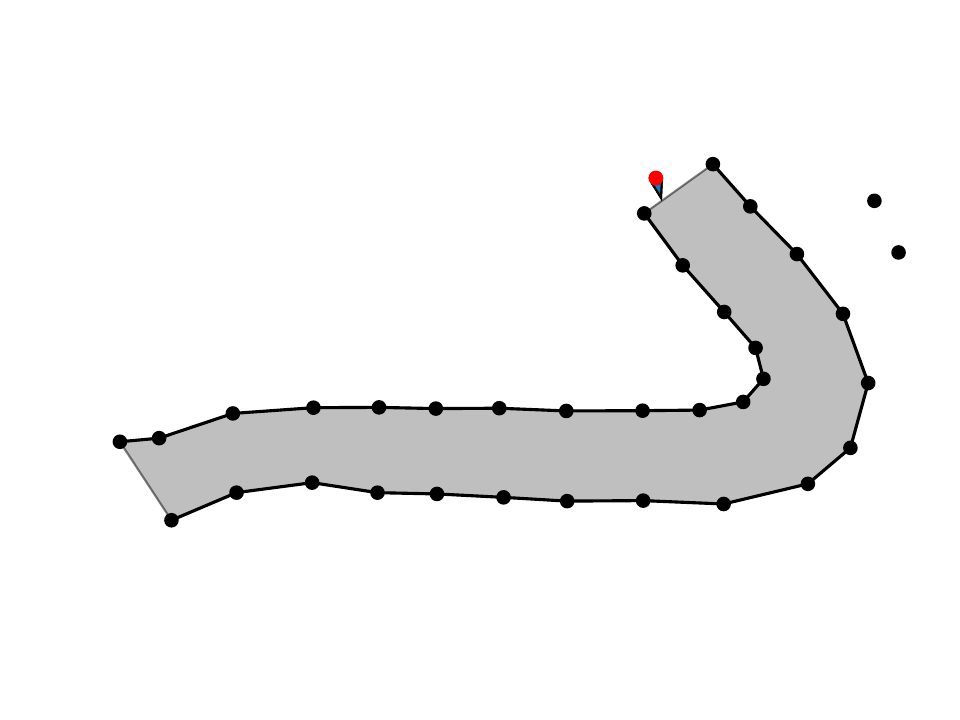}
		\end{adjustbox}
		\begin{adjustbox}{width=0.12\linewidth}
			\includegraphics{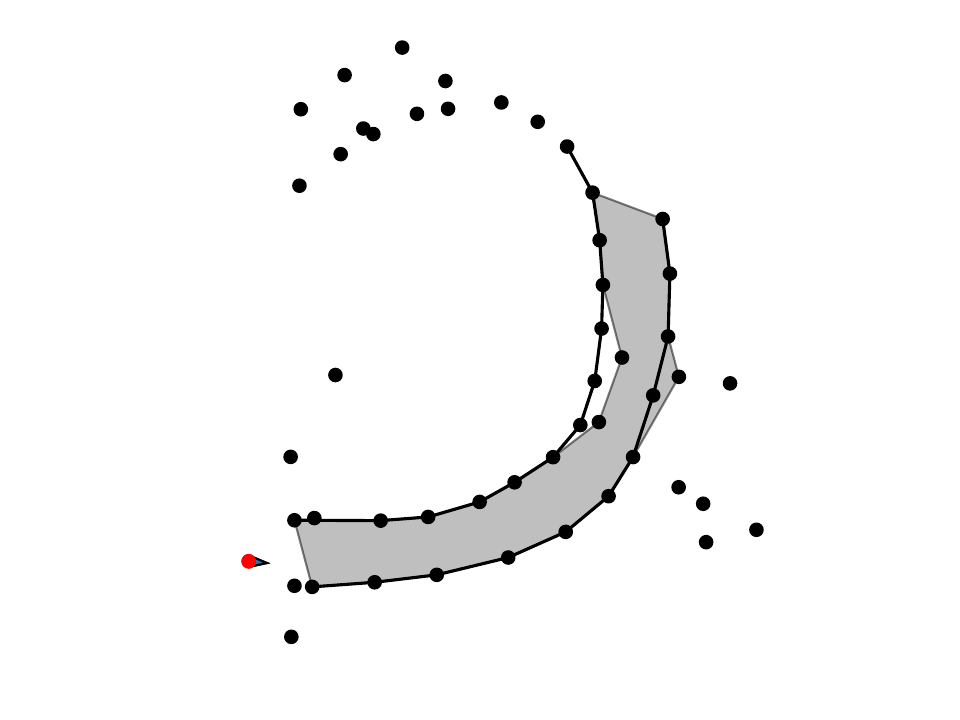}
		\end{adjustbox}
		\begin{adjustbox}{width=0.12\linewidth}
			\includegraphics{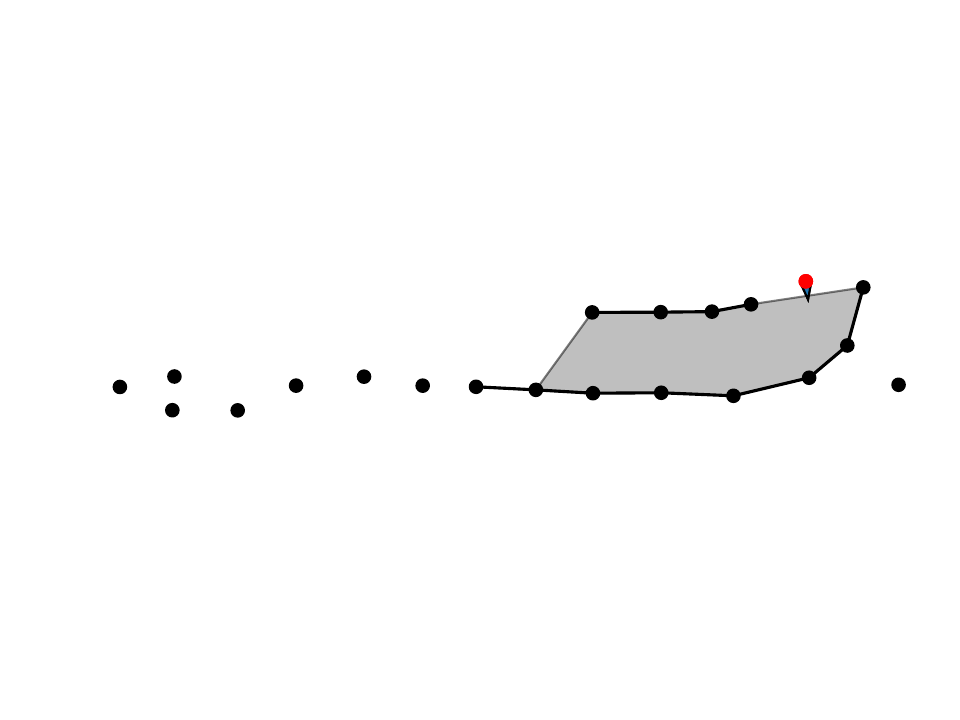}
		\end{adjustbox}
		\begin{adjustbox}{width=0.12\linewidth}
			\includegraphics{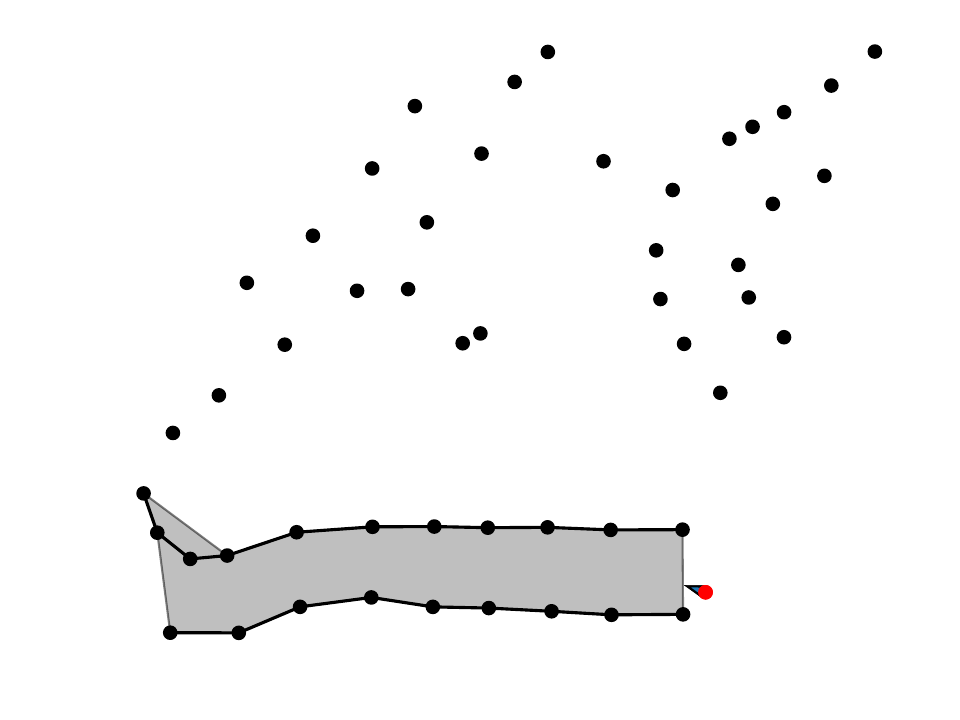}
		\end{adjustbox}
		\begin{adjustbox}{width=0.12\linewidth}
			\includegraphics{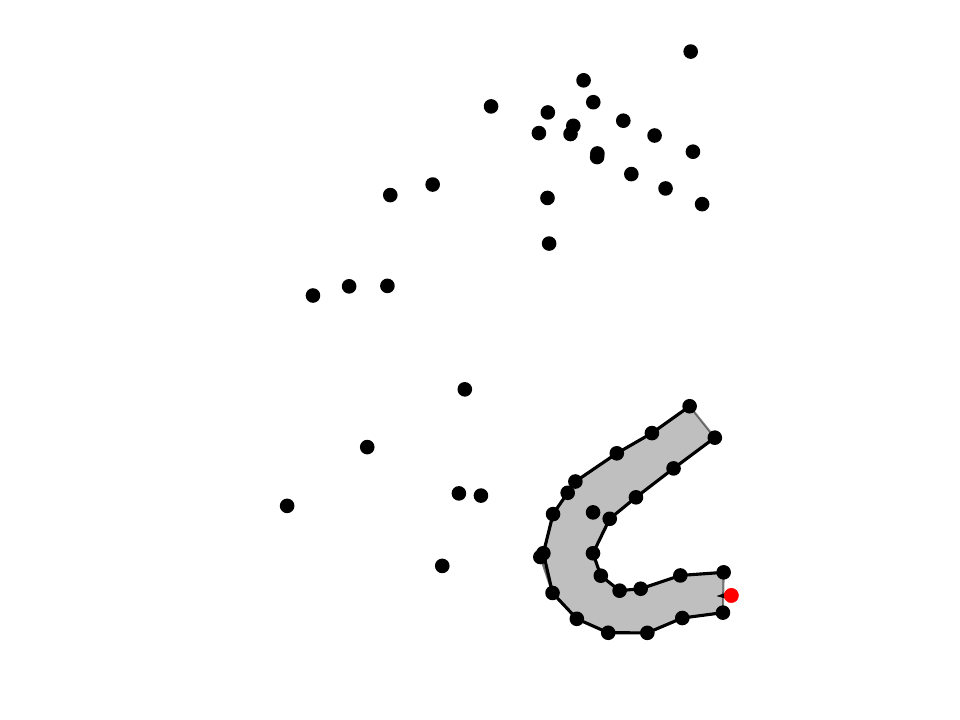}
		\end{adjustbox}
		\begin{adjustbox}{width=0.12\linewidth}
			\includegraphics{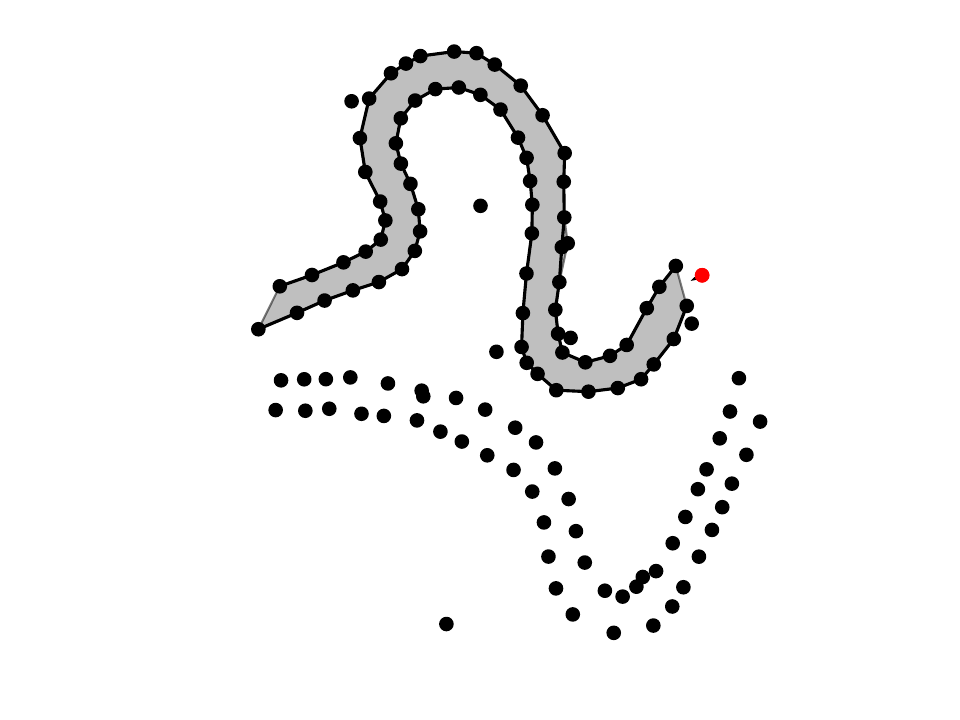}
		\end{adjustbox}
		\begin{adjustbox}{width=0.12\linewidth}
			\includegraphics{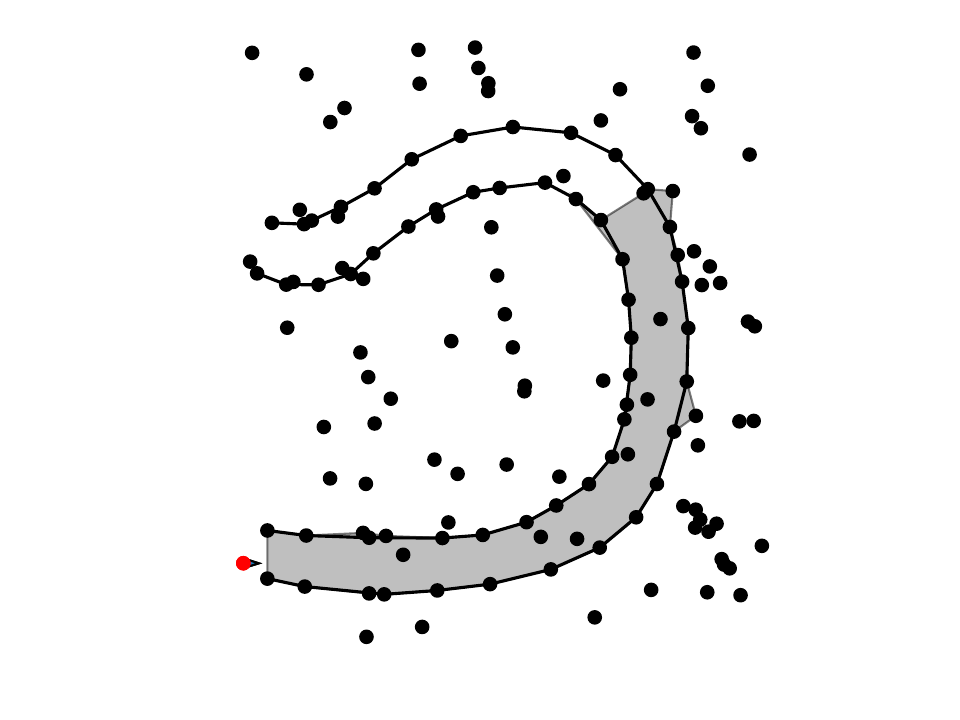}
		\end{adjustbox}
		\begin{adjustbox}{width=0.12\linewidth}
			\includegraphics{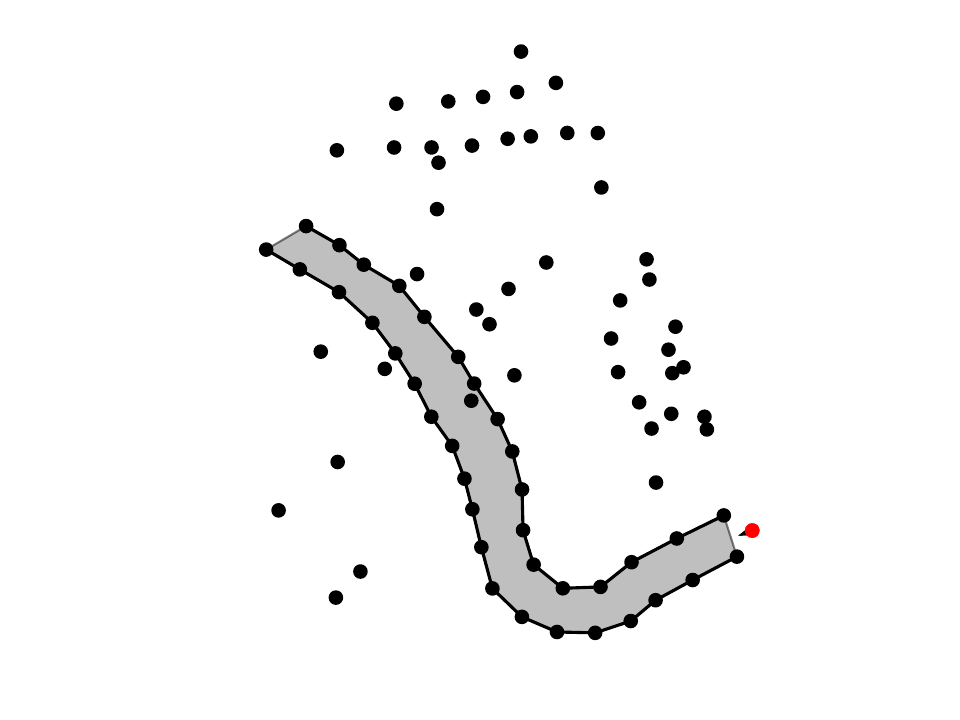}
		\end{adjustbox}
		\begin{adjustbox}{width=0.12\linewidth}
			\includegraphics{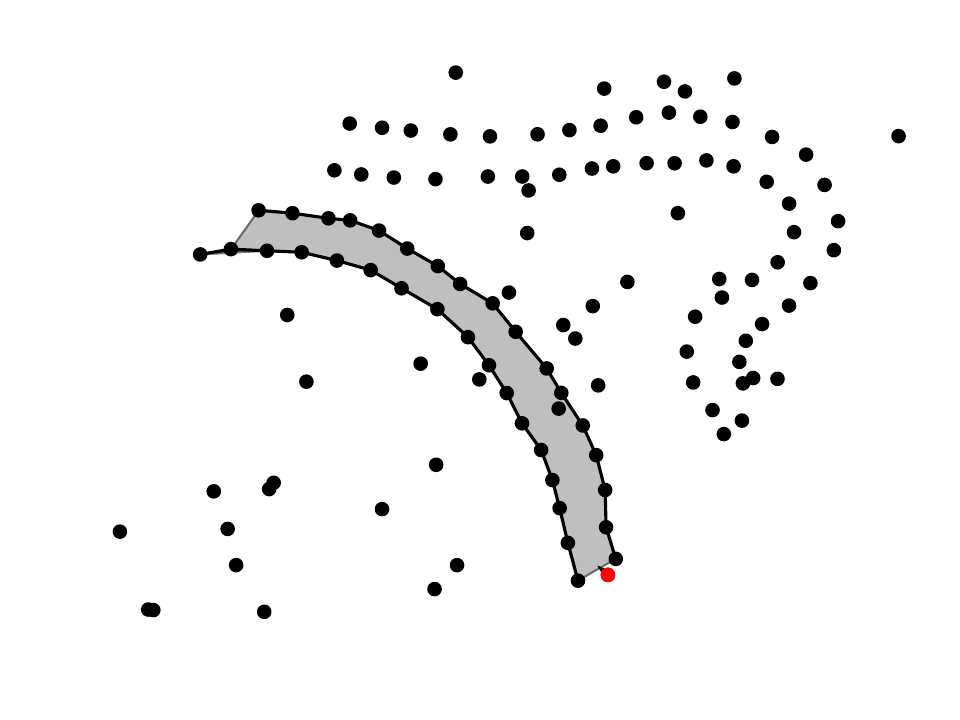}
		\end{adjustbox}
		\begin{adjustbox}{width=0.12\linewidth}
			\includegraphics{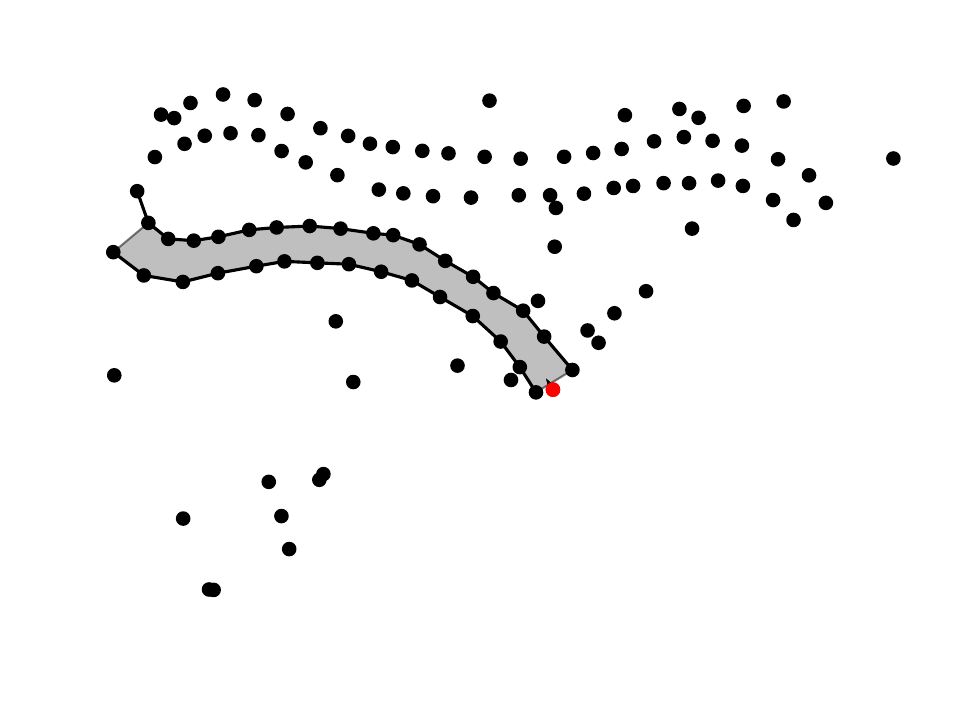}
		\end{adjustbox}
		\begin{adjustbox}{width=0.12\linewidth}
			\includegraphics{figures/case_2022-07-17_Trackdrive_20-10-47_355_success_iou_1_00}
		\end{adjustbox}
		\begin{adjustbox}{width=0.12\linewidth}
			\includegraphics{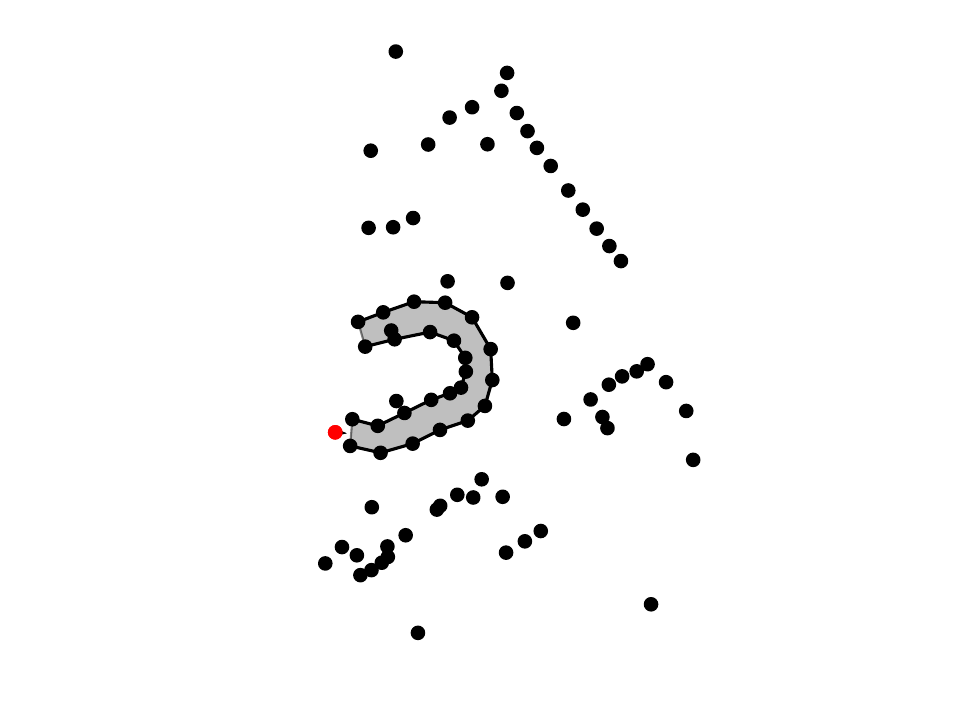}
		\end{adjustbox}
		\begin{adjustbox}{width=0.12\linewidth}
			\includegraphics{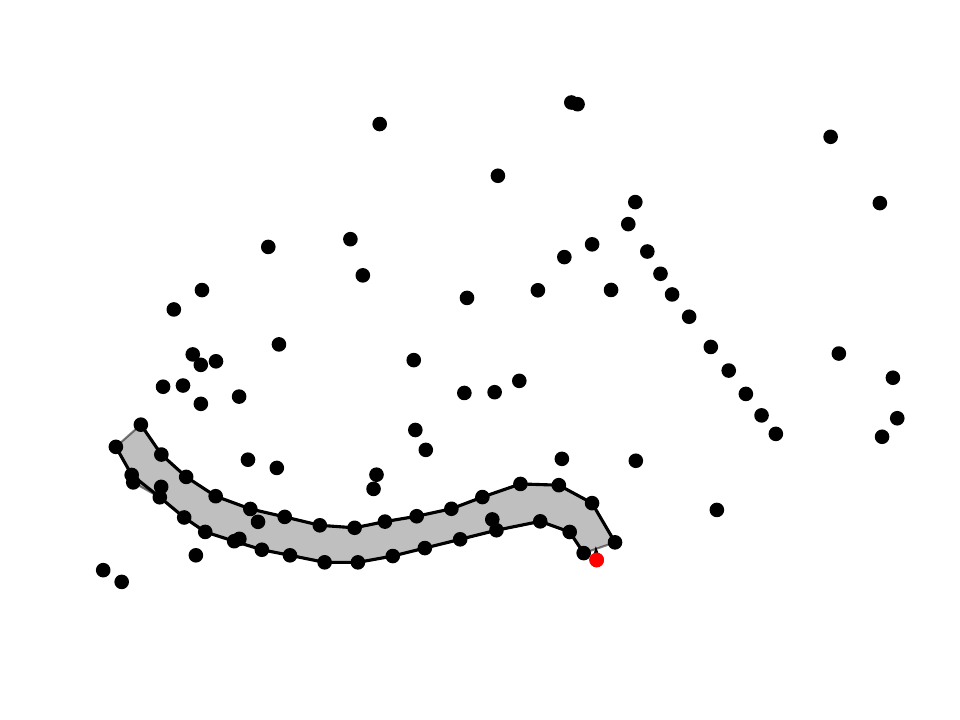}
		\end{adjustbox}
		\begin{adjustbox}{width=0.12\linewidth}
			\includegraphics{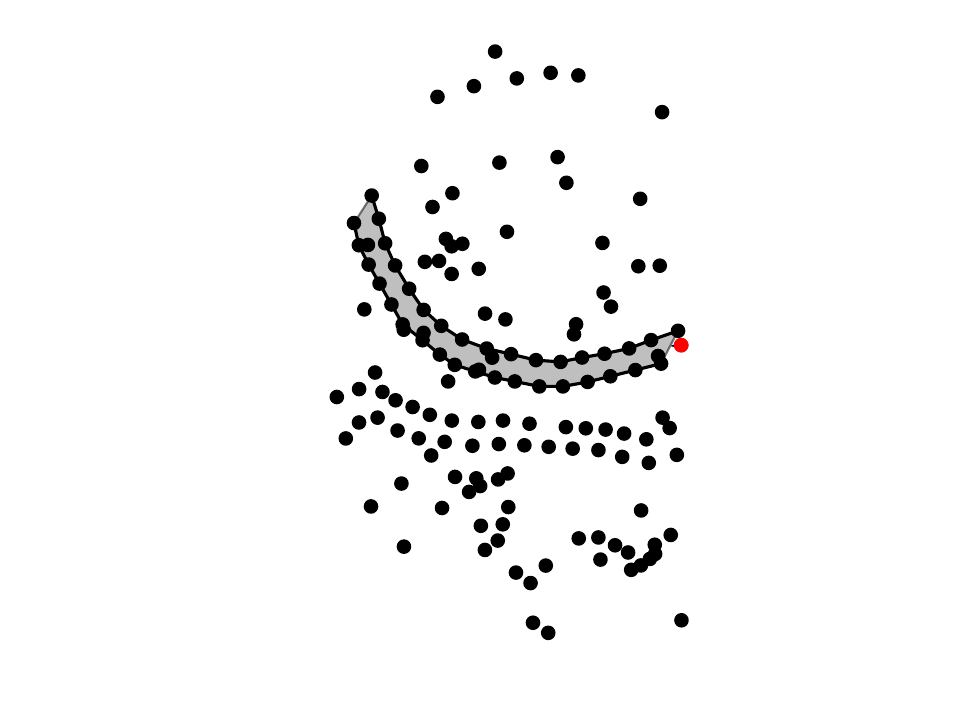}
		\end{adjustbox}
		\begin{adjustbox}{width=0.12\linewidth}
			\includegraphics{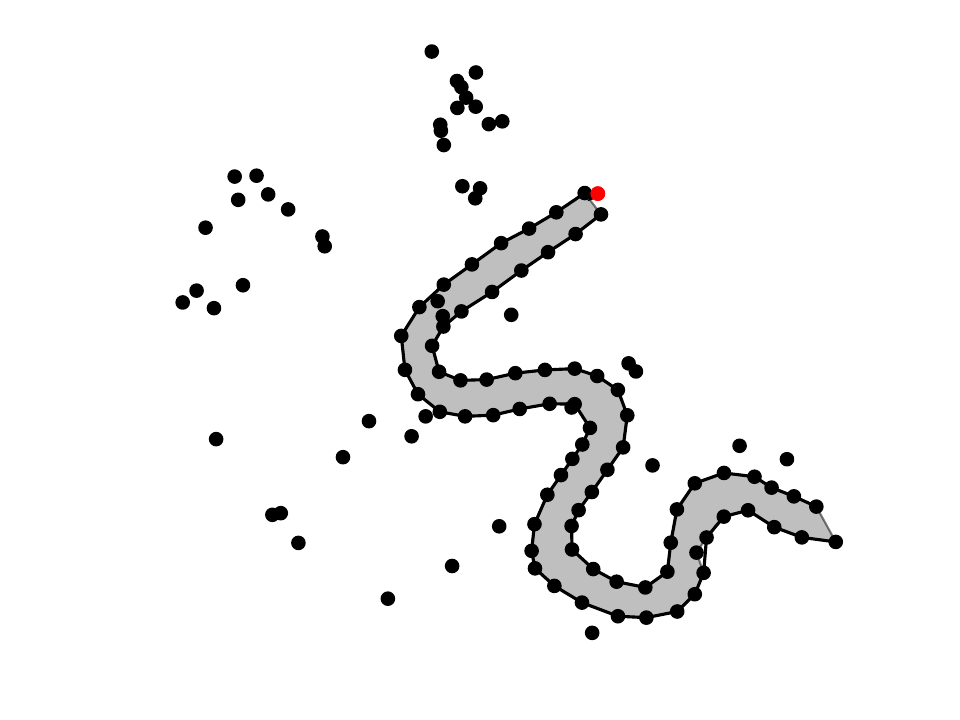}
		\end{adjustbox}
		\begin{adjustbox}{width=0.12\linewidth}
			\includegraphics{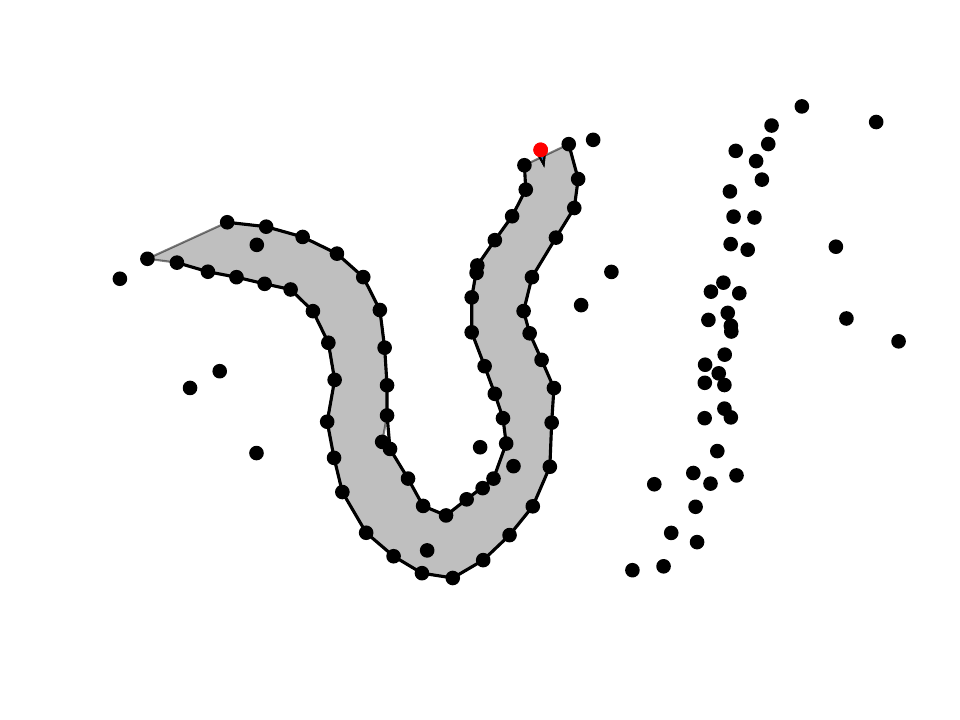}
		\end{adjustbox}
		\begin{adjustbox}{width=0.12\linewidth}
			\includegraphics{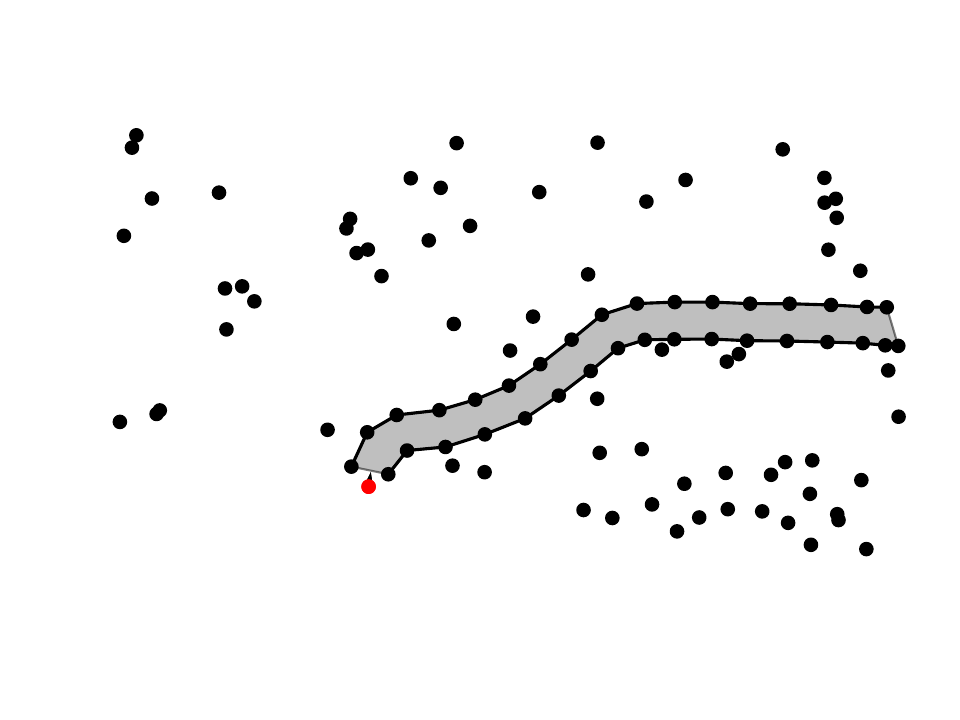}
		\end{adjustbox}
		\caption{Qualitative results of successful lane detection: Our algorithm is robust to high amounts of false-positives, disconnected parts of the lane and is able to detect curvy, long-horizon lanes.}
		\label{fig:qualitative-results}
	\end{figure*}
	
	Table \ref{tab:results} shows the rates of various prediction categories at different false-positive rates. The rate of predicting a lane diverging below 20m is overall very low and below 1\% for FP-rates of 0\% - 30\%. Our rate of such critical failures is 0.6\%, which is lower than the method from \cite{9197098} where 5\% are achieved.
	At 50\% FP-rate, which normally occurs only with bad performing perception, the critical failures rise to 2.6\%, which is similar to the approach presented in \cite{9197098}.	See also Fig. \ref{fig:drivinglanelengths-leaving-gt} for a histogram of the divergence distances.
	Compared to the methods presented in \cite{9197098} and \cite{9341702}, our algorithm generally predicts much longer driving lanes (Fig. \ref{fig:lane-lengths-predicted}) of up to 100m. In \cite{9341702}, the algorithm is limited to 15m, whereas in \cite{9197098} the prediction horizon is not reported but from the limitation to 16 input points and the constrained distance between two points (which is 5m), it can be concluded that the maximum lane length can be 40m (in case of no false-positives and a straight lane). Qualitative results are shown in Fig. \ref{fig:qualitative-results}.

	\subsection{Enumeration Evaluation} 
	\label{sec:enumeration-evaluation}

	\begin{table}[htb]
		\centering
		 \begin{tabular}{lccccccccc}
			\toprule
			\textbf{FP-Rate} &  \textbf{Instances with complete enumeration (\%)} \\
			\midrule
			0 \% & 46.4\\
			10 \% & 34.4\\
			30 \% & 13.4\\
			50 \% & 2.3\\
			\bottomrule
		\end{tabular}
		\caption{Candidate enumeration completeness: Especially at a low false-positive rate, the solution set is completely enumerated in many instances, allowing finding the global optimum of the likelihood function.}
		\label{tab:enumeration-completeness}
	\end{table}

	\begin{figure}[htb]
		\centering
		\begin{adjustbox}{width=.99\linewidth,center}
			\includegraphics{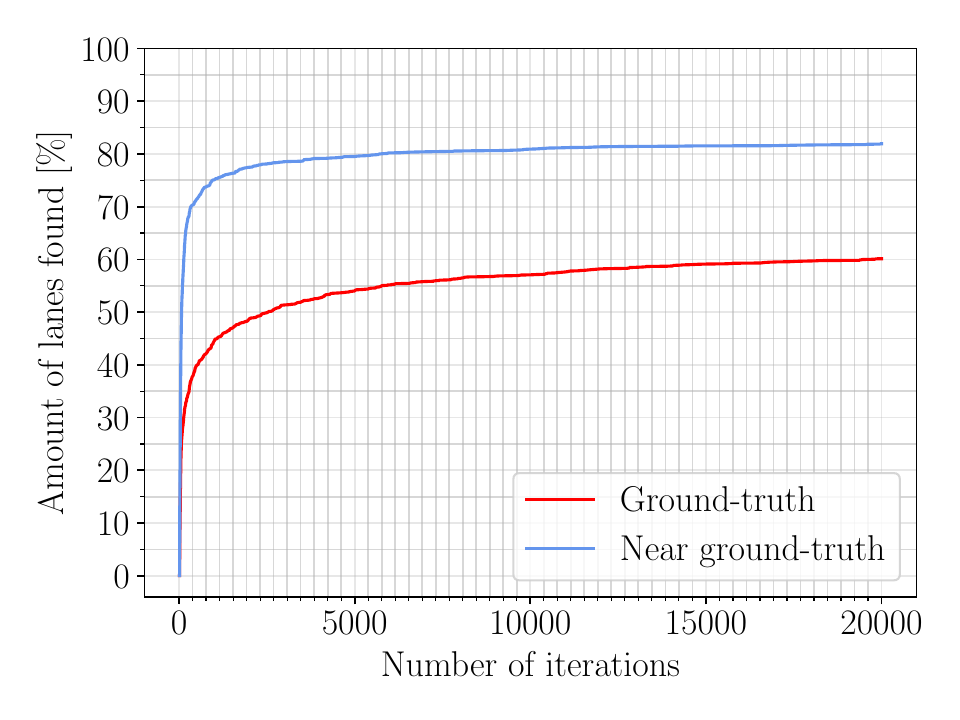}
		\end{adjustbox}
		\caption{Iterations needed to enumerate (near) ground truth lane candidates: Most are discovered during the initial 500 iterations, demonstrating effective search heuristics. Increasing the maximum number of iterations only yields diminishing returns in terms of the amount of correct candidates found, therefore setting it to \DefaultMaxIterations{} is sufficient while ensuring real-time operation. Note that although the enumeration is complete, a plateau is visible as the search becomes stuck in local minima -- lanes remain which are not found even after 20000 iterations and would require a vastly higher iteration count. (Shown for FP-rates of 0\% to 30\%)}
		\label{fig:bestrankedneedediterations}
	\end{figure}
	
	\begin{figure}[htb]
		\centering
		\begin{minipage}{0.241\textwidth}
			\begin{adjustbox}{width=\linewidth,center}
				\includegraphics{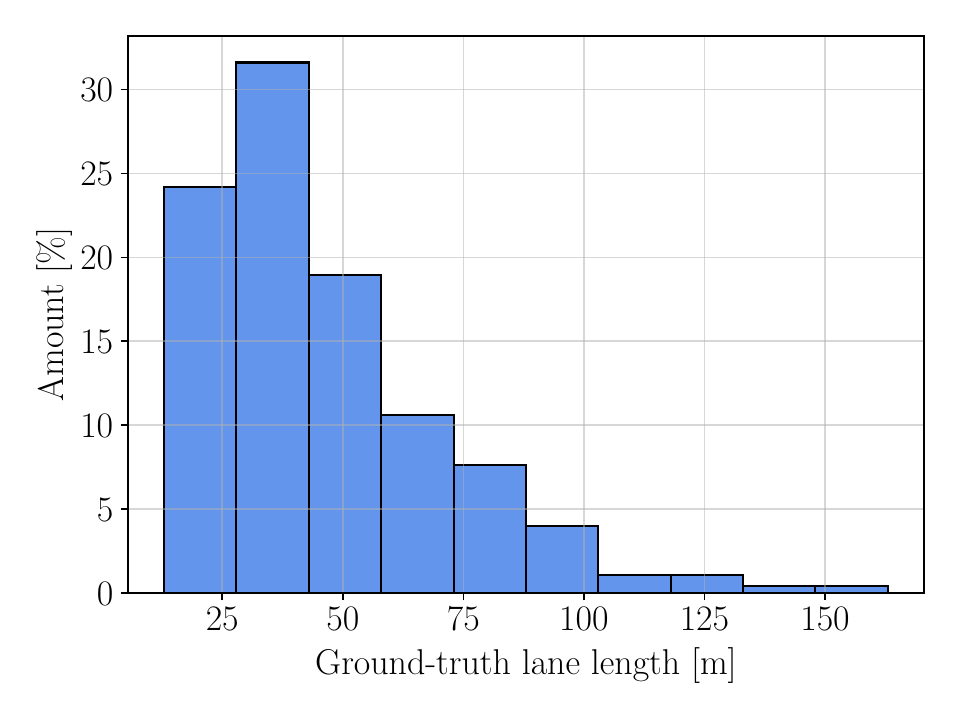}
			\end{adjustbox}
			\subcaption[first caption.]{Histogram of ground-truth lane lengths.}\label{fig:lane-lengths-gt}
		\end{minipage}
		\begin{minipage}{0.241\textwidth}
			\begin{adjustbox}{width=\linewidth,center}
				\includegraphics{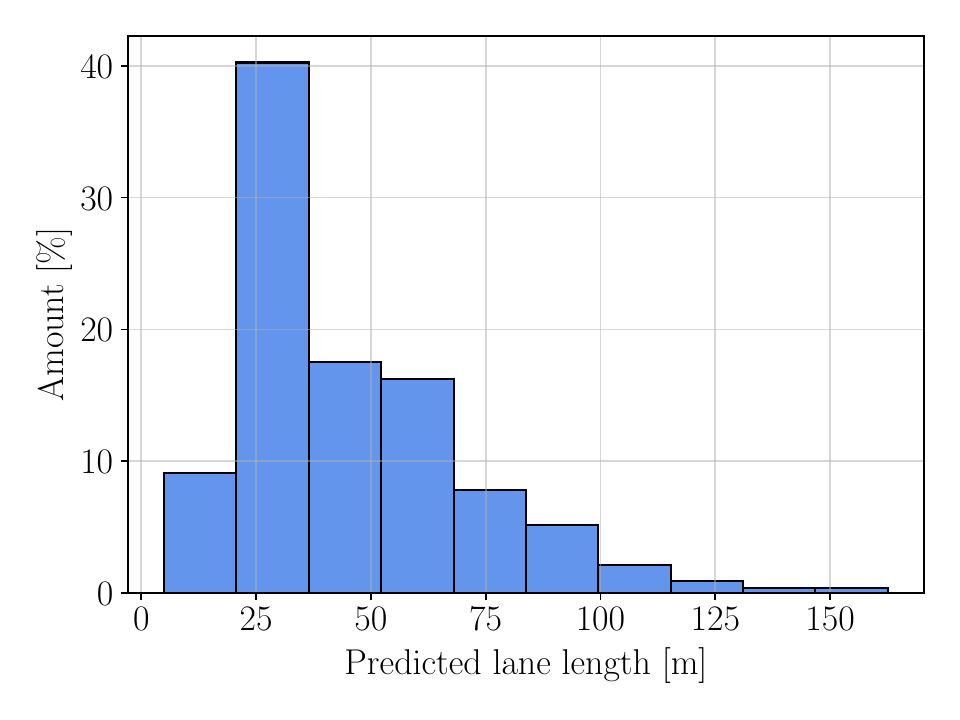}
			\end{adjustbox}
			\subcaption[first caption.]{Histogram of predicted lane lengths.}\label{fig:lane-lengths-predicted}
		\end{minipage}%
		\caption{(a): Histogram of the ground-truth driving lane lengths occurring in the dataset. The histogram of the length of non-divergent predictions (b) coincides -- very long driving lanes of 100m and even 150m are predicted similarly well. (Shown for FP-rates of 0\% to 30\%)} 
		\label{fig:lane-lengths}
	\end{figure}

	If an incorrect driving lane is predicted, we distinguish between two possible causes: either the correct driving lane was not found during enumeration, or the ranking selected the wrong candidate as the most likely one (the second step).
	In the following, we focus on the first failure mode and evaluate the frequency of enumeration failures. First, we observe that especially with no false-positives, the enumeration is in many instances complete (Tab. \ref{tab:enumeration-completeness}) -- this allows finding the global optimum of the likelihood function, without relying on heuristics.
	In cases of incomplete enumeration, we rely on heuristics to find a near ground-truth candidate with as few iterations as possible. We define a near ground-truth candidate as one that has at least 98\% IoU with the ground-truth driving lane. Figure \ref{fig:bestrankedneedediterations} shows that in most instances, a ground-truth or a near-ground-truth lane is enumerated within \DefaultMaxIterations{} iterations, demonstrating effectiveness of the search heuristics. At the same time, this analysis suggests a suitable value for the parameter $it_{max}$, the maximum number of iterations. Failure to enumerate a lane with at least 98\% IoU occurs in 5\% of instances with FP-rates of 0\% with 30\%.
	
	\begin{figure}[htb]
		\centering	
		\begin{adjustbox}{width=0.49\linewidth}
			\includegraphics{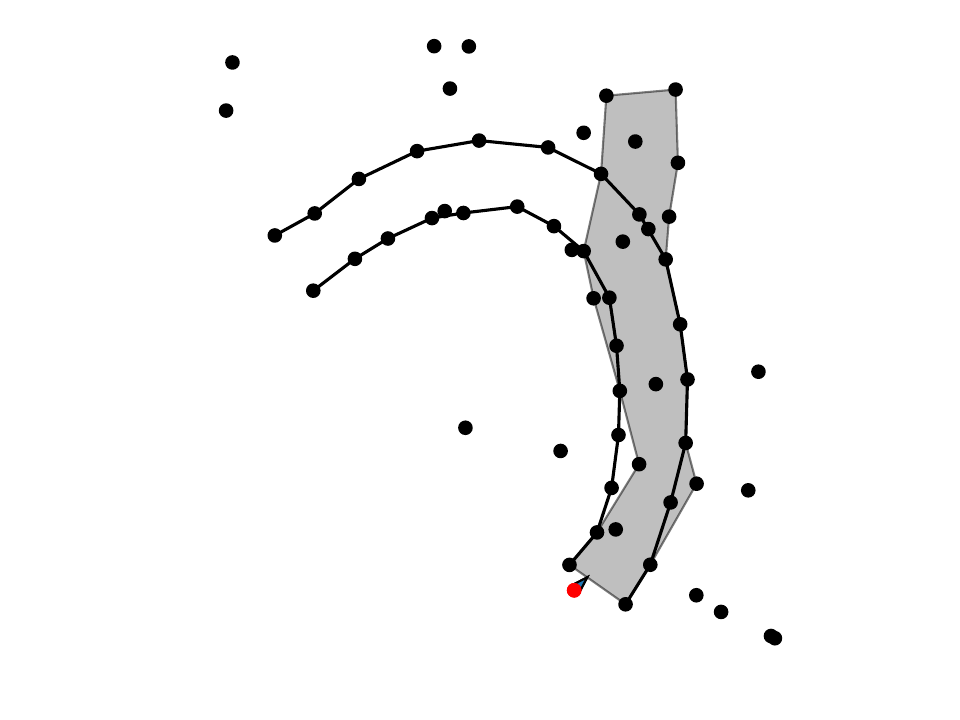}
		\end{adjustbox}
			\begin{adjustbox}{width=0.49\linewidth}
		\includegraphics{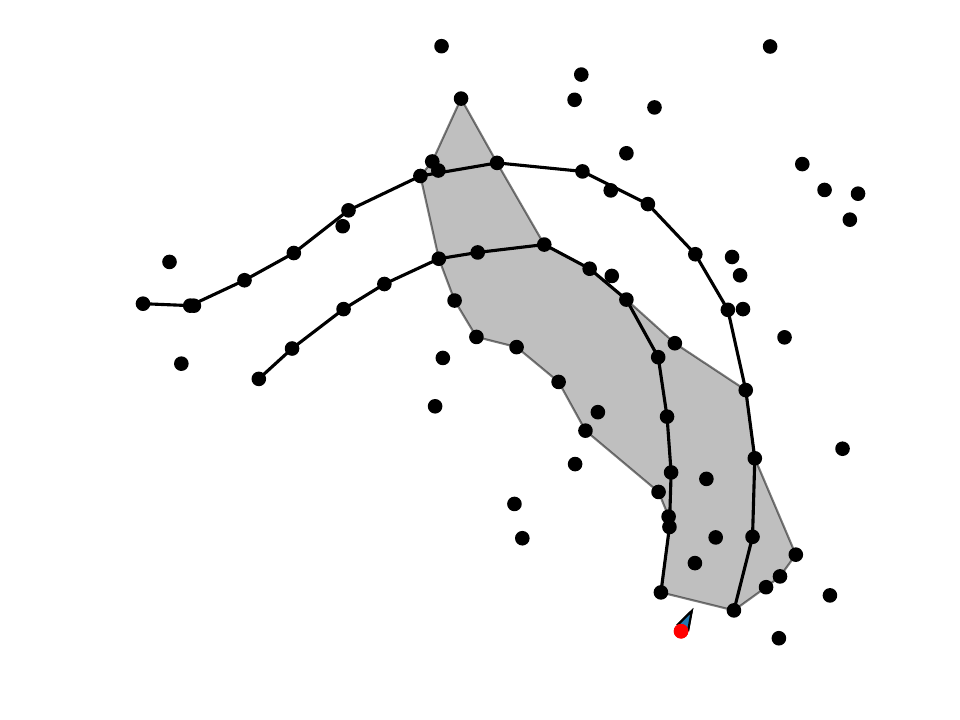}
	\end{adjustbox}
		\caption{Typical enumeration failure cases caused by greedy heuristics: The heuristic $\code{NVD}$ (see \ref{sec:nvd-def}) predicts straights (left) and does not favor parallel line segments (right). } 
		\label{fig:failure-cases-enumeration}
	\end{figure}

	Examples of enumeration failures are shown in Fig. \ref{fig:failure-cases-enumeration}. These failures are generally caused by greedy heuristics, leading to the search becoming stuck in a local minimum.
	Also, they are amplified by a high amount of false-positives. Specifically, failure mode (a) occurs as the heuristic \ref{sec:nvd-def} always predicts straight lanes. Such a failure mode may be reduced therefore by using a heuristic which predicts the curvature. Failure (b) occurs when the heuristic \ref{sec:nvd-def} selects the two next vertices independently, without considering the angle between the segments and thus without favoring parallel boundaries.
	
	\subsection{Lane ranking evaluation}
	
	\begin{figure}[htb]
		\centering	
		\begin{adjustbox}{width=0.49\linewidth}
			\includegraphics{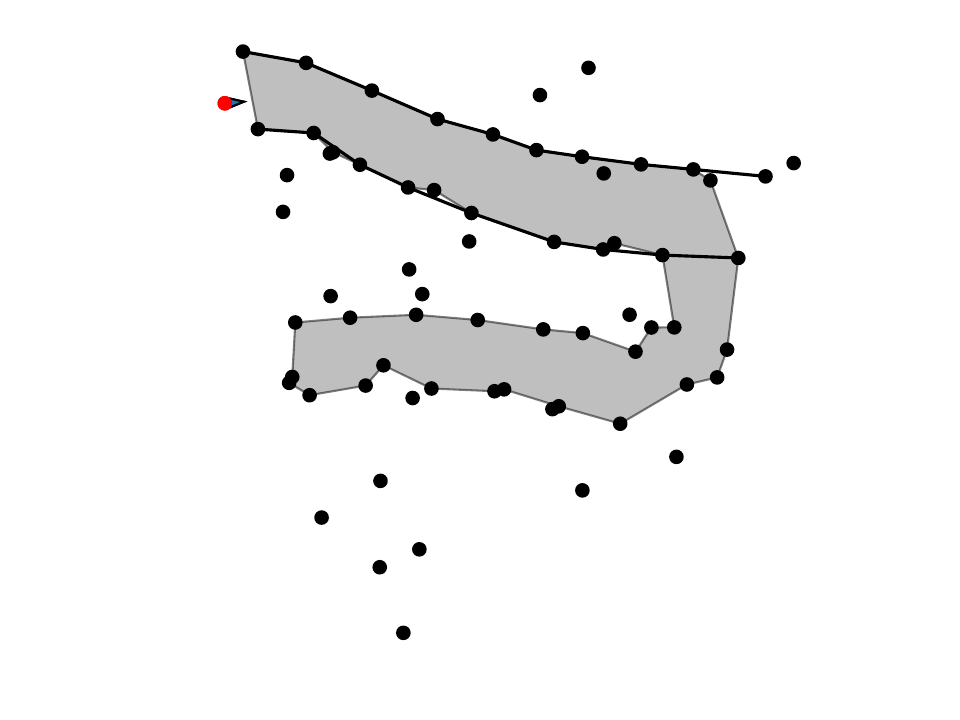}
		\end{adjustbox}
		\begin{adjustbox}{width=0.49\linewidth}
			\includegraphics{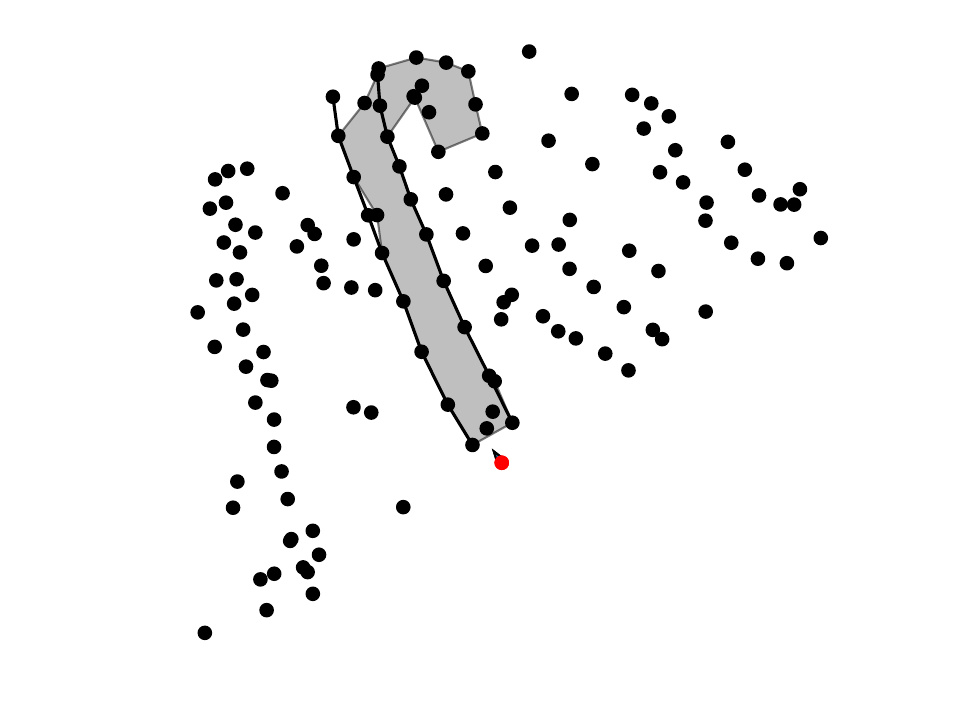}
		\end{adjustbox}
		\caption{Typical failure cases caused by a failure of the ranking to choose the most likely lane.} 
		\label{fig:failure-cases-ranking}
	\end{figure}
	
	We evaluated the ranking neural network on one-third of the dataset, while the remaining two-thirds were used for training, with false-positive rates of 0\% to 30\%. The mean Intersection over Union (mean-IoU) was used as the evaluation metric.
	The neural network achieved a mean-IoU of 95.10\% when evaluated on the complete dataset, including the training dataset. When evaluated on the hold-out evaluation dataset, the neural network achieved only a slightly lower mean-IoU of 94.79\% which demonstrates the ability to generalize well. Note that the maximum achievable IoU is not 100\% due to the enumeration not being able to always find the ground-truth candidate. Instead the maximum achievable IoU were 97.99\% and 97.75\% for the complete and the validation dataset respectively.
	
	We analyzed failure cases of the ranking (IoU $<$ 90\%), they are mainly diverging lanes, examples are shown in Fig. \ref{fig:failure-cases-ranking}. Most failure cases arise due to false-positives introducing ambiguity.
	
	\subsection{Execution time}
	\label{sec:exe-time-eval}

	\begin{figure}[htb]
		\centering
		\begin{adjustbox}{width=.7\linewidth,center}
			\includegraphics{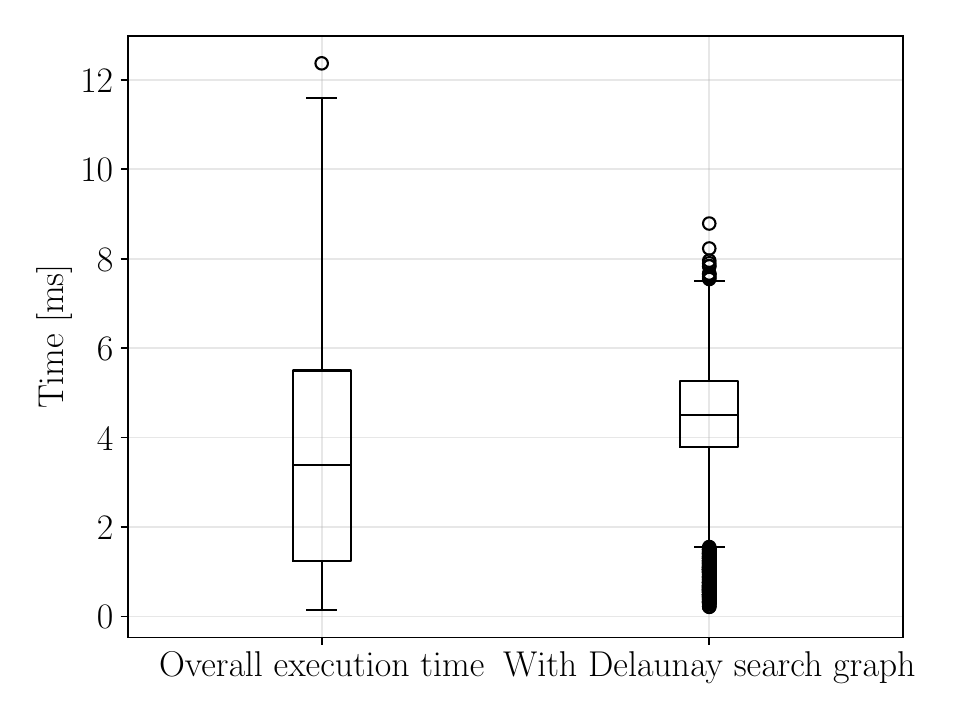}
		\end{adjustbox}
		\caption{Overall execution time of the proposed algorithm. Using the Delaunay-triangulation for constructing the search graph does not reduce the runtime significantly although the graph is theoretically sparser. (Sample size for each box-plot is 1964)}
		\label{fig:executiontime}
	\end{figure}
	
	\begin{figure}[htb]
		\centering
		\begin{adjustbox}{width=.99\linewidth,center}
			\includegraphics{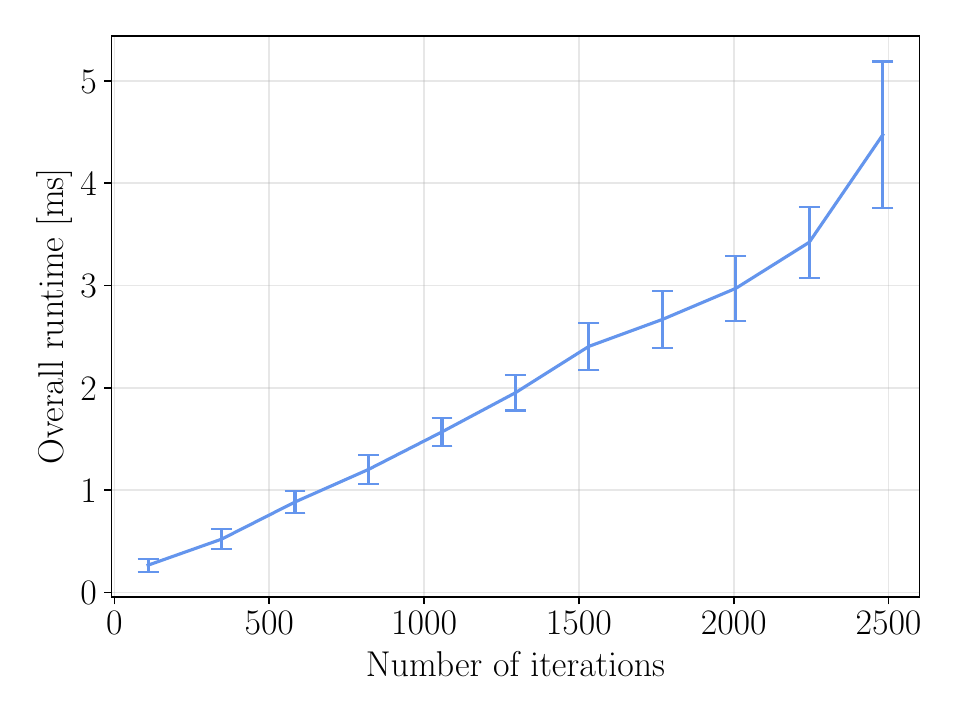}
			\includegraphics{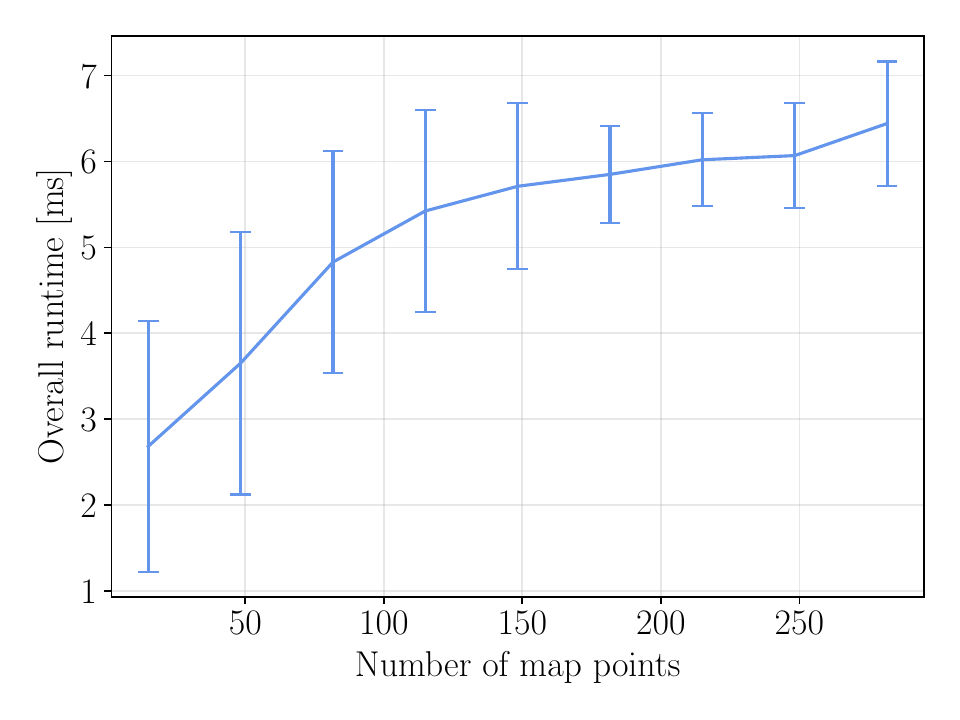}
		\end{adjustbox}
		\caption{Runtime of CLC over the number of iterations and number of map points respectively. The runtime of CLC increases linearly with the number of iterations whereas with the number of map points the runtime increases underproportional. Overall, the number of map points has a low correlation with the runtime. (The error bars show the interval of one standard deviation, sample size is 1964).}
		\label{fig:executiontime-over-iterations}
	\end{figure}	

	The overall execution time of CLC is shown in Fig \ref{fig:executiontime}. With a median runtime of under 5ms and a maximum runtime of 12ms, our algorithm CLC is able to meet the low-latency criterion required in autonomous racing. It is also able to run after each map update which happens at the LiDAR-rate of 20 Hz. When evaluating the runtime over the number of iterations performed and the number of input map points, we observe an approximately linear increase of the runtime when the iterations increase, whereas the number of map points only affects the runtime slightly (Fig. \ref{fig:executiontime-over-iterations}).
	Compared to the approach presented in \cite{9197098} where runtimes of 71ms - 176 ms are reported, our algorithm is faster while being able to handle an input containing much more than 16 points in the input map since the runtime does not grow over-proportionally with the number of map points. We also evaluated the algorithm using search graphs constructed by Delaunay-triangulation and otherwise the same setup, where we see no significantly lower runtime. This indicates that the search is not significantly more efficient despite the higher graph sparsity. By adjusting the maximum number of iterations, we can effectively limit the execution time of CLC and guarantee its low latency.
	
	Overall, the graph search is very efficient since it explores solutions only once and uses effective heuristics. The geometric constraints are also not too prohibitively expensive to compute when implemented with online-algorithms and short-circuiting is used.
	
	\section{Conclusion and further work}

	We presented a solution to the lane detection problem where the boundaries are sparsely marked only by 2D-points instead of continuous lines. By modeling both boundaries as paths in a graph and using exhaustive enumeration, we were able to hypothesize all possible driving lanes and thus be robust to high levels of false positives.
	By incorporating carefully designed geometric constraints on the candidate lanes, we were able to make the graph search efficient by pruning most of the solution set via backtracking.
	We presented simple and intuitive heuristics that quickly guide the graph search to locally optimal solutions,  allowing a near-ground-truth solution to be found in over 50\% of instances using less than 500 iterations of DFS.
	
	Our comprehensive evaluation on many different racetrack layouts, both easy and difficult with tight hairpin turns, shows superiority in both (1) the prediction horizon, where our approach is able to detect lanes over 100m in length and (2) in the robustness to ambiguous track layouts arising from partial map observation and false positives. Furthermore, our machine learning based approach for ranking candidate lanes effectively solves the pattern recognition problem and eliminates the need for manual tuning of heuristics. Overall, the proposed lane detection algorithm allows our racecar to drive at speeds up to 20m/s on previously unknown race tracks due to the long horizon prediction.
	The presented algorithm was used in the car of our Formula Student Team StarkStrom Augsburg, which participates in the most prestigious 
	Formula Student Germany competition, driving the fastest lap times from 2021 to 2023 and even winning the competition in 2022.

	Further work may include improving the graph search by (1) learning search heuristics, e.g with graph neural networks such as in \cite{8968113}, and (2) using exploration-vs-exploitation heuristics to help escape sub-optimal local minima. Another improvement may be to use boundary points directly as an input to the neural network instead of using hand-crafted features.
	
	\section*{Acknowledgements}
	We would like to thank the whole formula student team StarkStrom Augsburg for their support in maintaining and improving the race car and helping gather real data during test drives which made the evaluation possible.
	
	\bibliographystyle{IEEEtran}
	\bibliography{bibliography} 
	
	\begin{IEEEbiographynophoto}{Ivo Ivanov}
		Ivo Ivanov received his B.Sc degree in computer science from the Technical University of Applied Sciences Augsburg (Germany) in 2022 and is currently a master's student in computer science. He was an active member of the Formula Student team StarkStrom Augsburg from 2019 to 2021. Currently he focuses on research in perception and machine learning for autonomous driving.
	\end{IEEEbiographynophoto}
	\begin{IEEEbiographynophoto}{Carsten Markgraf}
		Prof. Dr.-Ing. Carsten Markgraf finished his studies of Electrical Engineering in 1997 at the University of Hannover (Germany). As a research assistant he was cooperating with VW to analyze infrastructure-based technologies for the automation of vehicle test tracks. After he finished his PhD on the topic "Automated Driving with Magnetic Nails" in 2002 he worked in the industry at ThyssenKrupp Automotive, first in Munich to develop intelligent chassis control systems and then in Liechtenstein to develop Electric Power Steering Systems (EPS). There he was responsible for the development and industrialization of the electrical, electronic and software components. After the first system went into mass production for the X3 BMW in 2010, he received the call from the Technical University of Applied Sciences Augsburg as professor in electrical engineering, responsible for Automatic Control and Dynamic Systems. Since then, he is also Faculty Advisor of the local Formula Student Team and leader of the research group DriverlessMobility. Since 2023 he is scientific director of the Technology Transfer Center Landsberg a. Lech, responsible for the Autonomous Systems.
	\end{IEEEbiographynophoto}
	
	\input{appendix.tex}

\end{document}

%% file: custom_commands.tex
\newcommand{\code}[1]{\texttt{#1}} 
\DeclareMathOperator*{\argmin}{arg\,min} 

\newcommand{\Set}[1]{\mathcal{#1}}
\newcommand{\EmptySet}{\emptyset}

\newcommand{\Rtwo}{\mathbb{R}^2}

\newcommand{\PathPair}{P}

\newcommand{\ConeMap}{\Set{M}}
\newcommand{\Edges}{\Set{E}}
\newcommand{\Vertices}{\Set{V}}
\newcommand{\SearchGraph}{\Set{G}}

\newcommand{\CarPose}{\mathbf{P}_{car}}

\newcommand{\SetLeftPaths}{\Set{L}}
\newcommand{\SetRightPaths}{\Set{R}}
\newcommand{\SolutionSet}{\Set{S}}
\newcommand{\FeasibleSolutionSet}{\Set{S}_c}

\newcommand{\NextSide}{s}

\newcommand{\EnumeratedPathPairs}{\Phi}

\newcommand{\LeftSpline}{L}
\newcommand{\RightSpline}{R}
\newcommand{\LanePolygon}{D}

\newcommand{\MatchinPoints}{M_{ps}}
\newcommand{\LaneWidth}{W}

\newcommand{\NextVertexDecider}{\code{NVD}}

\newcommand{\ConstratintDecider}{\code{CD}}
\newcommand{\BacktrackingDecider}{\code{BTD}}

\newcommand{\MinLaneWidth}{w_{min}}
\newcommand{\MaxConeDistance}{d_{max}}
\newcommand{\MaxLaneWidth}{w_{max}}
\newcommand{\MaxSemgentsAngle}{\phi_{max}}
\newcommand{\ConstraintLaneWidth}{C_{width}}
\newcommand{\ConstraintSegments}{C_{seg}}

\newcommand{\ConstraintSimplePoly}{C_{poly}}

\newcommand{\DefaultMaxIterations}{2500}

\theoremstyle{definition}
\newtheorem{definition}{Definition}
\newtheorem{theorem}{Theorem}
\newtheorem{lemma}[theorem]{Lemma}
\theoremstyle{remark}

\newcommand{\DatasetNumTracks}{9}


%% file: summary-table.tex
 \begin{tabular}{lccccccccc}
	\toprule
	\textbf{FP-Rate} & \multicolumn{4}{c}{\textbf{Success-Cases (\%)}} & \multicolumn{1}{c}{\textbf{Failure-cases (\%)}} & \textbf{Sum: Success (\%)} & \textbf{Failure (\%)} & \textbf{mean-IoU (\%)} \\
	\cmidrule(lr){2-5} \cmidrule(lr){6-6}
	& GT & Near-GT & Too short & Diverging $\geq20$m & Diverging $<20$m \\
	\midrule
0 \% & 34.5 & 61.9 & 0.7 & 2.2 & 0.6 & 99.4 & 0.6 & 96.9\\
10 \% & 24.9 & 70.9 & 0.5 & 2.9 & 0.8 & 99.2 & 0.8 & 96.0\\
30 \% & 9.3 & 85.1 & 0.8 & 4.1 & 0.7 & 99.3 & 0.7 & 93.2\\
50 \% & 2.1 & 70.9 & 2.6 & 19.3 & 5.0 & 95.0 & 5.0 & 79.5\\
	\bottomrule
\end{tabular}

%% file: appendix.tex
\appendix[Correctness of the enumeration algorithm]

\begin{lemma}[Backtracking criterion for $\ConstraintSegments$-constraint]
	\label{lemma:bt-criterion}
	The segment angle constraint $\ConstraintSegments$ can no longer be satisfied by adding more points to either boundary once it is violated.
\end{lemma}

\begin{proof}
	
	Given the two boundary polygonal chains $\LeftSpline$ and $\RightSpline$ represented by the path pair $P$ and defined by the point sequence $(l_1, ..., l_n)$ and $(r_1, ..., r_m)$ respectively, the constraint $\ConstraintSegments(P)$ is defined as follows:
	
	\begin{equation}
		\begin{aligned}
			\ConstraintSegments
			(P) = \bigwedge_{i=1}^{N-2} | \angle( \overline{ l_{i} l_{i+1} }, \overline{ l_{i + 1} l_{i + 2} }) |  < \MaxSemgentsAngle{} \\ \land  \bigwedge_{i=1}^{M - 2} | \angle( \overline{ r_{i} r_{i+1} }, \overline{ r_{i + 1} r_{i + 2} }) |  < \MaxSemgentsAngle
		\end{aligned}    
	\end{equation}
	
	By adding more vertices, more line segments are added and thus more consecutive angles are computed and therefore more boolean variables are added over which the conjunction is evaluated.
	Since a conjunction that is false cannot become true by adding more boolean variables to it, the $\ConstraintSegments$ constraint can no longer be satisfied by adding more points to either boundary once it is violated.
\end{proof}

\begin{lemma}[Backtracking criterion $\ConstraintSimplePoly$-constraint]
	\label{lemma:polygon-simplicity-bt-criterion}
	The	constraint $\ConstraintSimplePoly$ requiring the driving lane polygon to be simple cannot be satisfied anymore by extending the lane except if it is violated by the line segment between the last points of the two boundaries intersecting any other line segment.
\end{lemma}

\begin{proof}
	The driving lane polygon $\LanePolygon$ is defined by the sequence of left boundary points concatenated with the right boundary points in reversed order: $\LanePolygon = (l_1, ..., l_n, r_m, ..., r_1)$. A polygon is simple if no line segment of its boundary intersects any other (non-adjacent) line segment.
	
	The line segment $\overline{l_n r_m}$ between the last points of the two boundaries changes every time the lane is extended (new points are added), since new points become the last ones. If therefore $\overline{l_n r_m}$ intersects any other segment, this intersection may be undone. All other edges of the polygon do not change as the lane is extended, therefore if any two segments intersect and none of them is $\overline{l_n r_m}$, this intersection cannot be undone and $\ConstraintSimplePoly$ cannot be satisfied anymore.
\end{proof}

\begin{lemma}[Backtracking criterion $\ConstraintLaneWidth$-constraint]
	\label{lemma:constraint-lw-bt-criterion}
	If $\ConstraintLaneWidth$ is violated, it may be satisfied by extending the lane depending on the following three cases:
	\begin{enumerate}
		\item If it is violated by a fixed matching line being too long or too short, it cannot be satisfied anymore
		\item If it is violated by a mutable matching line being too short, it cannot be satisfied anymore
		\item If it is violated by a mutable matching line being too long, it may be satisfied by extending the boundaries
	\end{enumerate}
\end{lemma}

\begin{proof}
	
	To show the first case, we observe that the online algorithm may only add more matching points to set of fixed matching points (see Alg. \ref{alg:online-lane-width}, line 8). 
	Therefore, only more boolean variables are added to constraint Eq. \ref{eq:lane-width-constraint-def} and since a conjunction that is false cannot become true by adding more boolean variables to it, the constraint cannot be satisfied anymore once it is violated.\\
	
	For the remaining two cases, observe that we can rewrite constraint \ref{eq:lane-width-constraint-def} as:
	\begin{align}
		\ConstraintLaneWidth(P) = \min\{w_i \mid w_i \in \LaneWidth(L, R)\} > \MinLaneWidth \label{eq:lw-constraint-min}\\
		\land \max\{w_i \mid w_i \in \LaneWidth(L, R)\} < \MaxLaneWidth \label{eq:lw-constraint-max}
	\end{align}

	If the constraint violation is due to a matching line which belongs to the mutable set, it may be recomputed, we thus need to consider the computation of matching points (Eq. \ref{eq:lane-width-nn-search}) which is based on nearest neighbor search. Consider the general case of a nearest neighbor search between the query and target sets $\Set{Q} \subset \Rtwo$ and $\Set{T} \subset \Rtwo$:
	
	\begin{align}
		NN(\Set{Q}, \Set{T}) := \{ \min_{t_i \in \Set{T}} \{|| q_i - t_i||_2\} \mid q_i \in \Set{Q}\} 
	\end{align}
	
	The following inequalities hold: 
	
	\begin{align}
		\max{NN(\Set{Q}, \Set{T})} \leq \max{NN(\Set{Q'}, \Set{T})}, \Set{Q} \subset \Set{Q'} \subset \Rtwo \label{eq:nn-max-larger-query}\\
		\min{NN(\Set{Q}, \Set{T})} \geq \min{NN(\Set{Q'}, \Set{T})}, \Set{Q} \subset \Set{Q'} \subset \Rtwo \label{eq:nn-min-larger-query}\\
		\max{NN(\Set{Q}, \Set{T})} \geq \max{NN(\Set{Q}, \Set{T'})}, \Set{T} \subset \Set{T'} \subset \Rtwo \label{eq:nn-max-larger-target} \\
		\min{NN(\Set{Q}, \Set{T})} \geq \min{NN(\Set{Q}, \Set{T'})}, \Set{T} \subset \Set{T'} \subset \Rtwo \label{eq:nn-min-larger-target}
	\end{align}

	As the boundaries are extended, the query as well as the target domains of the NN-search becomes larger (Eq. \ref{eq:lane-width-nn-search} and \ref{eq:matching-points}). If the $\ConstraintLaneWidth$-constraint is therefore violated by a matching line that belongs to the mutable set being too short (Eq. \ref{eq:lw-constraint-min}), from equations \ref{eq:nn-min-larger-query} and \ref{eq:nn-min-larger-target} it follows that it cannot become longer and therefore the constrain cannot be satisfied anymore by extending the boundaries, concluding the second case.
	
	If however the $\ConstraintLaneWidth$-constraint is violated by a mutable matching line that is too long (Eq. \ref{eq:lw-constraint-max}), from equation \ref{eq:nn-max-larger-target} it follows that it may become shorter by extending the boundaries and therefore the constrain may be satisfied, concluding the third case.
\end{proof}

\begin{theorem}[Backtracking does not exclude valid solutions \ref{alg:enumerate_path_pairs_rec}]\label{thm:correctness-enumeration}
	Backtracking in Algorithm \ref{alg:enumerate_path_pairs_rec} does not exclude candidates which would satisfying constraints.
\end{theorem}

\begin{proof}
	Algorithm \ref{alg:enumerate_path_pairs_rec} backtracks in line 22, i.e. it does not extend the boundaries further. 
	By lemmata \ref{lemma:bt-criterion}, \ref{lemma:polygon-simplicity-bt-criterion} and \ref{lemma:constraint-lw-bt-criterion}, backtracking does not exclude candidates which would satisfy the constraints and thus the enumeration is complete given a high enough iteration limit.
\end{proof}